\documentclass[a4paper]{article}
\usepackage[preprint, nonatbib]{neurips_2022}

\usepackage{amsmath}
\usepackage{mathrsfs}
\usepackage{amsthm}
\usepackage{amssymb}
\usepackage{enumerate}
\usepackage{tikz}

\usepackage[
    backend=biber,
    citestyle=numeric-comp,
    bibstyle=ieee,
    sorting=nyt,
    isbn=false,
    url=false,
    citecounter=true,
    ]{biblatex}
\AtEveryBibitem{\clearlist{language}}

\DeclareSourcemap{
  \maps[datatype=bibtex]{
    \map{
      \step[fieldsource=title,
        match=\regexp{\\\$},
        replace=\regexp{\$}]
      \step[fieldsource=title,
        match=\regexp{\{\\textasciicircum\}},
        replace=\regexp{\^}]
      \step[fieldsource=title,
        match=\regexp{\\\_},
        replace=\regexp{\_}]
      \step[fieldsource=title,
        match=\regexp{\{\\textbackslash\}},
        replace=\regexp{\\}]
      \step[fieldsource=title,
        match=\regexp{\\\{\{R\}\\\}},
        replace=\regexp{\{R\}}]
    }
  }
}

\usepackage[hidelinks]{hyperref}
\usepackage[acronym]{glossaries}
\usepackage[capitalize]{cleveref}

\usepackage{float}

\newacronym{dnn}{DNN}{deep neural network}
\newacronym{pde}{PDE}{partial differential equation}
\newacronym{pinn}{PINN}{physics informed neural network}
\newacronym{snn}{SNN}{shallow neural network}

\crefformat{equation}{#2(#1)#3}
\crefname{figure}{Figure}{Figures}
\Crefname{figure}{Figure}{Figures}

\newcommand{\nn}[1]{\mathbf{N}^{#1}}
\newcommand{\rr}[1]{\mathbf{R}^{#1}}
\newcommand{\rrp}[1]{\mathbf{R}_+^{#1}}
\newcommand{\cc}[1]{\mathbf{C}^{#1}}
\newcommand{\rrI}{\rr{d_1}}
\newcommand{\rrII}{\rr{d_2}}
\newcommand{\rrIxII}{\rrI\times\rrII}

\newcommand{\zzp}[1]{\mathbf{Z}_+^{#1}}

\newcommand{\scal}[2]{\langle #1,#2\rangle}
\newcommand{\eabs}[1]{\langle #1\rangle}     %%%%%   for <x>
\newcommand{\abs}[1]{| #1|}     %%%%%   for |x|
\newcommand{\absbig}[1]{\left| #1\right|}     %%%%%   for |x|
\newcommand{\cdo}{\, \cdot \, }

\newcommand{\conv}[0]{*}
\newcommand{\norm}[2]{\|#1\|_{#2}}
\newcommand{\normbig}[2]{\left\|#1\right\|_{#2}}
\newcommand{\F}[0]{\mathscr{F}}
\newcommand{\FL}[1]{\F L^{#1}}
\newcommand{\chF}[1]{\chi_{#1}}
\newcommand{\ch}[2]{\chF{#1}(#2)}
\newcommand{\chEF}[2]{\chF{#1}^{#2}}
\newcommand{\chE}[3]{\chF{#1}^{#2}(#3)}

\numberwithin{equation}{section}

\newtheorem{theorem}{Theorem}[section]
\newtheorem{definition}[theorem]{Definition}
\newtheorem{corollary}[theorem]{Corollary}
\newtheorem{lemma}[theorem]{Lemma}
\newtheorem{example}[theorem]{Example}
\newtheorem{proposition}[theorem]{Proposition}

\theoremstyle{remark}
\newtheorem{remark}[theorem]{Remark}
\newtheorem*{remark*}{Remark}

\providecommand{\keywords}[1]
{
  \small	
  \textbf{Keywords} #1
}
\providecommand{\msc}[1]
{
  \small	
  \textbf{Mathematics Subject Classification} #1
}

\title
{Space-Time Approximation with Shallow Neural Networks
	in Fourier Lebesgue spaces}

\author{Ahmed Abdeljawad$^*$
        and Thomas Dittrich\thanks{Both authors contributed equally.}
		\\[1ex]
	Johann Radon Institute for Computational and Applied Mathematics (RICAM),\\
    Austrian Academy of Sciences,\\
    Altenbergerstr. 69, 4040 Linz, Austria
	\\
	\texttt{\{ahmed.abdeljawad, thomas.dittrich\}@oeaw.ac.at}}

\begin{document}
	
\maketitle
	
\begin{abstract}
    Approximation capabilities of \glspl{snn}
    form an integral part in understanding the properties of \glspl{dnn}.
    In the study of these approximation capabilities some very popular
    classes of target functions are the so-called spectral Barron spaces.
    This spaces are of special interest
    when it comes to the approximation of \gls{pde} solutions.
    It has been shown that the solution of
    certain static \glspl{pde} will lie in some spectral Barron space.
    In order to alleviate the limitation to static \glspl{pde} and
    include a time-domain that might have
    a different regularity than the space domain,
    we extend the notion of spectral Barron spaces
    to anisotropic weighted Fourier-Lebesgue spaces.
    In doing so, we consider target functions
    that have two blocks of variables,
    among which each block is allowed
    to have different decay- and integrability-properties.
    For these target functions we first study the
    inclusion of anisotropic weighted Fourier-Lebesgue spaces
    in the Bochner-Sobolev spaces.
    With that we can now also measure the approximation error
    in terms of an anisotropic Sobolev norm,
    namely the Bochner-Sobolev norm.
    We use this observation in a second step where we
    establish a bound on the approximation rate 
    for functions from the anisotropic weighted Fourier-Lebesgue spaces
    and approximation via \glspl{snn} in the Bochner-Sobolev norm.
\end{abstract}


\msc{
41A25,$\,\cdot\,$%
41A46,$\,\cdot\,$%
41A30,$\,\cdot\,$%
46E35,$\,\cdot\,$%
62M45,$\,\cdot\,$%
68T05
}
	
%%%%%%%%%%%%%%%%%%%%%%%%%%%%%%%%%%%%%%%%%%%%%%%%%%%%%%
%               Introduction                         %
%%%%%%%%%%%%%%%%%%%%%%%%%%%%%%%%%%%%%%%%%%%%%%%%%%%%%%
	
\section{Introduction}\label{sec:Introduction}
In recent years \glspl{dnn} have gained
a huge amount of attention not only in practical applications,
but also from a theoretical perspective \cite{Grohs22MathematicalAspectsDeep,DeVore21NeuralNetworkApproximation}.
Up to now the development of a rigorous theory
that explains the empirical success of \glspl{dnn} is an active field.
In this context the study of networks
with a single hidden layer is especially important
in order do build a foundation
for the understanding of more complex \glspl{dnn}
by understanding the properties of the individual layers.
The study of these \glspl{snn}
has recently (re-)gained a surge of attention
from several different perspectives
such as
the characterization of representable functions
and the associated representation cost
\cite{
Abdeljawad22IntegralRepresentationsShallow,
Chizat20ImplicitBiasGradient,
Savarese19HowInfiniteWidth,
Parhi21BanachSpaceRepresenter,
Ongie20FunctionSpaceView,
},
asymptotic approximation properties
\cite{
Bach17BreakingCurseDimensionality,
Caragea23NeuralNetworkApproximation,
E22RepresentationFormulasPointwise,
Klusowski16RiskBoundsHighDimensional,
Klusowski18ApproximationCombinationsReLU,
Siegel20ApproximationRatesNeural,
Siegel22HighOrderApproximationRates,
Siegel22SharpBoundsApproximation,
Ma22UniformApproximationRates,
Voigtlaender22SamplingNumbersFourierAnalytic,
},
and the application to \glspl{pde}
\cite{
Chen21RepresentationSolutionsElliptic,
Chen23RegularityTheoryStatic,
Gonon23RandomFeatureNeural,
Lu21PrioriGeneralizationAnalysis,
Lu22PrioriGeneralizationError,
Marwah23NeuralNetworkApproximations,
}.

In the context of approximation based on \glspl{snn} as hypothesis class,
a very popular family of target functions are the spectral Barron spaces
\cite{
Barron93UniversalApproximationBounds,
Breiman93HingingHyperplanesRegression,
Caragea23NeuralNetworkApproximation,
Chen23RegularityTheoryStatic,
E22BarronSpaceFlowInduced,
Klusowski18ApproximationCombinationsReLU,
Ma22UniformApproximationRates,
Siegel20ApproximationRatesNeural,
Voigtlaender22SamplingNumbersFourierAnalytic,
}.
For the present work,
these classes of functions are important out of two major reasons:
First, 
learning the parameters of a \gls{snn}
is a special case of function-representation
by means of a dictionary.
This means that the theory of dictionary learning (see 
\cite{
DeVore98NonlinearApproximation,
Kurkova01BoundsRatesVariableBasis,
Barron08ApproximationLearningGreedy,
}) can be applied to \glspl{snn}.
Thereby, one can develop bounds on the approximation error
based on the width $N$ of the network.
For functions in the spectral Barron spaces,
these bounds come without curse of dimensionality, i.e., the scaling in $N$ is independent of the input dimension $d$.
Such an approximation result for spectral Barron spaces was first shown by \citeauthor{Barron93UniversalApproximationBounds} in
\cite{
Barron93UniversalApproximationBounds,
}.
Second,
the spectral Barron spaces are especially interesting for application to \glspl{pde}
This is because for different types of \glspl{pde} it has been shown that under certain conditions the solution will be in a spectral Barron space 
(see for example \cite{Gonon23RandomFeatureNeural,Chen23RegularityTheoryStatic}).

As already mentioned, the notion of spectral Barron spaces dates back to the 1990s where \citeauthor{Barron93UniversalApproximationBounds} \cite{Barron93UniversalApproximationBounds} identified a class of functions that can be approximated without curse of dimensionality by a shallow network with sigmoidal activation function.
This class of functions $f$ is characterized by the $L^1$ integrability condition for the first Fourier-moment, i.e.,
\begin{align}
\label{eq:original_barron}
    \int_{\rr{d}}\abs{\xi}\abs{\hat{f}(\xi)}d \xi<\infty,
\end{align}
where $\hat{f}$ denotes the Fourier transform of $f$.
Throughout the literature,
this weighted spectral $L^1$ integrability condition has been modified
in several different ways with regard to the weighting.
As seen above, in the seminal work of \citeauthor{Barron93UniversalApproximationBounds} \cite{Barron93UniversalApproximationBounds}
the choice for the weight function was $\omega(\xi)=\abs{\xi}$.
Other choices for polynomially bounded weights that have been proposed in the literature are
$\omega(\xi)=\abs{\xi}^2$
in \cite{Breiman93HingingHyperplanesRegression},
$\omega(\xi)=\abs{\xi}_1^s$ with $s\in\{2,3\}$
in \cite{Klusowski16RiskBoundsHighDimensional,Klusowski18ApproximationCombinationsReLU},
$\max\{1,\abs{\xi}^2\}$ 
in \cite{Gonon23RandomFeatureNeural},
$\omega(\xi)=(1+\abs{\xi})^s$ with any $s>0$
in \cite{Siegel20ApproximationRatesNeural,Siegel22HighOrderApproximationRates,Siegel22SharpBoundsApproximation,Siegel23CharacterizationVariationSpaces,Voigtlaender22SamplingNumbersFourierAnalytic,Xu20FiniteNeuronMethod},
$\omega(\xi)=(1+\abs{\xi}^2)^{\frac{s}{2}}$ with $s>0$
in \cite{Chen23RegularityTheoryStatic} and with $s=2$
in \cite{E22RepresentationFormulasPointwise},
and
$\omega(\xi)=\sup_{x\in\Omega}\abs{\left<\xi,x-x_0\right>}$ for some $x_0$ in the domain $\Omega$ of $f$
in \cite{Caragea23NeuralNetworkApproximation}.
The exponentially bounded weight 
$\omega(\xi)=e^{c\abs{\xi}^\beta}$ with $0<\beta<1$ and $0<c$ 
has been treated in \cite{Siegel22HighOrderApproximationRates}.

A common rationale in the definition of spectral Barron spaces
is to allow functions that are only defined on some bounded domain $\Omega$.
This has the consequence that different Fourier representations
lead to the same function when restricting to the given domain.
In order for a function to be in the spectral Barron space,
there needs to be at least one representation for which the integrbility condition is finite.
Thus, it is desirable to consider the infimum over all representation.
A unified formulation with a general weight function
that also allows the possibility of bounded domains can be written as follows:
\begin{definition}[Spectral Barron Space with General Weight]
\label{def:spectral_Barron_space}
    Let $d\in\nn{}$, $\Omega\subseteq\rr{d}$, $\omega:\rr{d}\to (0,\infty)$ such that $\Omega$ and $\omega$ are a measurable set and function, respectively.
    For a function
    $f\in L^1(\Omega)$
    the spectral Barron (semi-)norm with general weight is given by
    \begin{align*}
        \norm{f}{\mathscr{B}_{\omega}(\Omega)}:=\inf_{\substack{f_e\in L^1(\rr{d})\\f_e|_\Omega=f}}\int_{\rr{d}}\omega(\xi)\abs{\hat{f}_e(\xi)}d\xi
    \end{align*}
    and the spectral Barron space with general weight is given by
    \begin{align*}
        \mathscr{B}_{\omega}(\Omega):=\{f\in L^1(\Omega):\norm{f}{\mathscr{B}_{\omega}(\Omega)}<\infty\}.
    \end{align*}
\end{definition}
Note that in general $\norm{\cdot}{\mathscr{B}_{\abs{\cdot}}(\Omega)}$ is
only a semi-norm,
whereas $\norm{\cdot}{\mathscr{B}_{(1+\abs{\cdot})^s}}$ is indeed a norm \cite{Siegel22SharpBoundsApproximation}.

We want to note that the term
\emph{Barron space} (without the prefix \emph{spectral}) is also used for a second, different,
class of functions.
These \emph{infinite-width Barron spaces} are defined via representations 
in terms of infinite-width shallow networks
\cite{Caragea23NeuralNetworkApproximation,Li20ComplexityMeasuresNeural,E22BarronSpaceFlowInduced}.
For \glspl{snn} with heaviside activation and ReLU activation
it can be shown that the spectral Barron space
with some specific choices of weights
is a subset of the infinite-width Barron space (cf. \cite[Lemma 7.1]{Caragea23NeuralNetworkApproximation}).
However, in the present work we deal with generalizations
of the spectral Barron space and extensions for these embeddings
into the infinite-width Barron space are left for future work.

One of the central elements of approximation theory for \glspl{snn} is $N$-term approximation from dictionaries.
According to \citeauthor{DeVore98NonlinearApproximation} \cite{DeVore98NonlinearApproximation}, this was first introduced by \citeauthor{Schmidt07ZurTheorielinearen} \cite{Schmidt07ZurTheorielinearen} in 1907.
A major contribution to the study of upper bounds
on the approximation error of $N$-term approximations
is attributed to some non published results of Maurey by \citeauthor{Pisier80RemarquesResultatNon} \cite{Pisier80RemarquesResultatNon}.
This result is as follows:
Functions in the closure of the convex hull
of some subset $\mathcal{G}$ of a type-$2$ Banach space $\mathcal{B}$
can be approximated well by linear combinations
of $N$ elements of $\mathcal{G}$ in the sense that 
the error in the $\mathcal{B}$-norm is bounded by $cN^{-1/2}$
with some constant $c>0$.
This statement can be found in \cite[Lemma 1]{Barron93UniversalApproximationBounds}
stated in terms of Hilbert spaces and in \cite[Theorem 1]{Siegel22SharpBoundsApproximation}
for general type-$2$ Banach spaces.
\citeauthor{Barron93UniversalApproximationBounds} \cite{Barron93UniversalApproximationBounds}
used this result to prove
that for every $N\in\nn{}$ and every $f$ with $\norm{f}{\mathscr{B}_{\abs{\cdot}}(\rr{d})}<\infty$
there is a \gls{snn} $f_N$ with $N$ neurons
and sigmoidal activation function, such that
\begin{align*}
    \norm{f-f_N}{L^2(\mu,\rr{d})}\lesssim N^{-\frac{1}{2}}
\end{align*}
for any probability measure $\mu$.
Recently, the class of functions to which the approximation result of Maurey applies
has been extended to the variation space with respect to dictionaries 
\cite{Bach17BreakingCurseDimensionality,Kurkova01BoundsRatesVariableBasis,Ma22UniformApproximationRates,Siegel22SharpBoundsApproximation,Siegel23CharacterizationVariationSpaces}.
This allowed the development of bounds
on the $L^2$ norm
\cite{Breiman93HingingHyperplanesRegression,Klusowski18ApproximationCombinationsReLU}
and the $H^m$ norm
(cf. \cite{Lu21PrioriGeneralizationAnalysis,Chen23RegularityTheoryStatic} for $m=1$
and \cite{Siegel20ApproximationRatesNeural,Siegel22HighOrderApproximationRates} for general $m\in\nn{}$)
in terms of the spectral Barron norm.
All these bounds apply to shallow networks with either ReLU$^k$ activation with $k=1$ \cite{Breiman93HingingHyperplanesRegression}
or with an arbitrary $k\in\nn{}$ \cite{Siegel22HighOrderApproximationRates},
cosine activation \cite{Siegel22HighOrderApproximationRates,Chen23RegularityTheoryStatic},
or some general polynomially decaying activation functions \cite{Siegel20ApproximationRatesNeural}.

For compact or smooth dictionaries and certain choices of parameters, the approximation rates can be further improved from $O(n^{-1/2})$ up to $O(n^{-1})$ for sigmoidal activation function with $L^2$ error \cite{Makovoz96RandomApproximantsNeural} and to $O(n^{-(k+1)})$ for ReLU$^k$ activation function of order $k$ with $L^2$ \cite{Siegel22HighOrderApproximationRates} and $L^\infty$ error \cite{Ma22UniformApproximationRates}.
However, in our setting we do not require compactness or smoothness
of the dictionary and, thus,
we restrict to approximations rates of order $O(n^{-1/2})$.
In our theory, we unify all the bounds with finite integrability-exponent 
in the error measure
by considering the Sobolev norm $W^{m,p}$
with $m\geq0$ and $2\leq p<\infty$
and by allowing arbitrary weight functions.

In the present literature on the application of \glspl{snn} to \glspl{pde},
spectral Barron spaces have been used to characterize the solutions of
the
static Schr\"{o}dinger equation 
\cite{Lu21PrioriGeneralizationAnalysis,Chen23RegularityTheoryStatic},
the Poisson equation
\cite{Lu21PrioriGeneralizationAnalysis},
nonlinear variational \glspl{pde} \cite{Marwah23NeuralNetworkApproximations},
Black-Scholes Type \glspl{pde}
\cite{Gonon23RandomFeatureNeural}.
This natural application of spectral Barron spaces
comes from the fact that the weight in the integrability condition
can be seen as the symbol of a certain \gls{pde}.
However, it can be seen that in the current applications,
the variables are constrained to be of the same order of differentiability.
Nevertheless, for example for evolution equations
it is beneficial to consider different orders of differentiability
for the time- and the space-variable (see \cite[Chapter 7]{Evans22PartialDifferentialEquations}).
In our work we build the foundation to extend the theory
and application of spectral Barron spaces
to more general classes of \glspl{pde}
by allowing a general weight $\omega$ that fulfills some mild conditions
and by allowing two blocks of variables (each of arbitrary size)
that come with different orders of differentiability.

\subsection{Contribution and Outline}
Inspired by the successfull application of spectral Barron spaces
to the approximation of solutions of \glspl{pde},
we have established a novel perspective in this field which is aimed at allowing
non-isotropic differentiability and integrability of these solutions.
The motivation for this is (as hinted already above) that
for time-dependent \glspl{pde},
it is often advantageous to consider the time variable separately
and allow different properties for the resulting two blocks of variables.
We achieve this goal by introducing the so-called
\emph{anisotropic weighted Fourier-Lebesgue spaces}
to the field of \glspl{snn}
as the natural extension
of the spectral Barron spaces.
The main focus of our investigation is on the
properties and the approximation capability of target functions
in the anisotropic weighted Fourier-Lebesgue spaces.
In what follows, we will cover the different branches of our contribution:

From the perspective of harmonic analysis,
we extended the existing theory for isotropic weighted Fourier-Lebesgue spaces
$\mathscr{F}L^p(\omega;\rr{d})$ (see \cite{Pilipovic11MicroLocalAnalysisFourier})
to the case of two blocks
$\mathscr{F}L^{p,q}(\omega; \rrI,\rrII)$ in \cref{def:weighted_FL_mixed}
(the multi blocks case is then a straightforward extension of this).
We build the cornerstone for these spaces
by combining ideas from anistropic Lebesgue spaces
and weighted Banach spaces.
The first result in this regard is an inclusion between
Bochner-Sobolev spaces and Fourier-Lebesgue spaces
in different situations.
Specifically for solutions to time-dependent \glspl{pde},
we showed the following:
\begin{lemma}
\label{lem:intro_FourierSobolevImbedding}
    Let
    $1\leq s_i,q_i\leq 2\leq p_i\leq \infty$ with $p_1\leq p_2$
    such that $\frac{1}{s_i}+\frac{1}{q_i}+\frac{1}{p_i}=2$,
    for $i\in\{1,2\}$.
    Let $\omega(x,y)$ be a weight function
    defined on $\rr{}\times\rr{d}$, elliptic with respect to
    $\eabs{t}^{n_1}\eabs{x}^{n_2}$,
    $\Omega \subseteq \rr{d}$ be a bounded space domain,
    and $I\subset \rr{}$ be a bounded time domain.
    Let \(u \in \FL{q_1,q_2}(\omega;\rr{},\rr{d})\)
	then
    we have
    \begin{equation}\label{res:Intro_high_degree_bound}
		\norm{u}{W_{n_1,p_1}^{n_2, p_2}(I, \Omega)}
        \lesssim
        \norm{{\chF{I}}}{\FL{s_1}(\rr{})}
            \norm{{\chF{\Omega}}}{\FL{s_2}(\rr{d})}
            \norm{u}{\FL{q_1,q_2}(\omega;\rr{},\rr{d})},
	\end{equation} 
    where the hidden constant in \cref{res:Intro_high_degree_bound}
    depends only on the regularity and the integrability
    of the solution $u$ with respect to time and space variables.
\end{lemma}
We refer the reader to \cref{lem:smoothness_lemma_high_degree}
for the general statement of \cref{lem:intro_FourierSobolevImbedding}.
It can be seen that the right side of \cref{res:Intro_high_degree_bound}
depends on the integrability of the characteristic function of both domains.
We discuss the importance of this and the dependency
on the geometry for more general domains in details in \cref{sec:structure_of_domain}.
For the particular choice
$I=[0,T]$ with $T>0$ for the time domain,
$\Omega = [-1,1]^d$ for the space domain,
and $s_1, s_2>1$ for the degree of integrability,
the inclusion inequality
\cref{res:Intro_high_degree_bound} can be reduced to
\begin{align}
    \norm{u}{W_{n_1,p_1}^{n_2, p_2}([0,T], [-1,1]^d)}
            \lesssim
            \left(\frac{4s_1}{T(s_1-1)}\right)^{\frac{s_1+1}{s_1}}
            \left(\frac{2s_2}{s_2-1}\right)^{d\frac{s_2+1}{s_2}}
            \norm{u}{\FL{q_1,q_2}(\omega;\rr{}, \rr{d})}.
\end{align}

Our approach has the potential for further expansion,
in the sense that our framework can be applied to various cases
such as Bochner-Besov spaces and more generally to Bochner-Banach spaces.
However, we chose to restrict our focus to the Bochner-Sobolev
case in order to streamline our computations and maintain a clear focus.
We leave these extensions for future research.

Furthermore, we contributed to the field of learning theory with \glspl{snn}
by providing a generalization of spectral Barron spaces
and studying the associated approximation properties.
As seen in \cref{def:spectral_Barron_space} and the preceeding discussion,
the existing work deals solely with a single block of variables
(i.e., $\mathscr{B}_\omega(\rr{d}) = \mathscr{F}L^1(\omega, \rr{d})$).
That is, either the \gls{pde} of interest is static (i.e., there is no time variable)
or the time- and space variables are stacked into a single joint variable.
The latter limits the analysis of the approximation error in the Sobolev norm
to the minimum degree of differentiability of the two blocks
and it enforces that both variables have to be integrable in the same $L^p$ norm.
In the present work we introduce a space-time
version of spectral Barron spaces and thereby,
we allow for much more
generality and precision when addressing
time dependent \glspl{pde}.
More generally, this also applies to the approximation
multi-variable functions with different
integrability-, growth- and regularity properties in each block of variables.
Additionally, the extension to Fourier-Lebesgue spaces
allows us to treat different types of integrability, namely, $L^q$ with $q\in [1,2]$.
By doing so, we obtain existing (single-block) approximation results
for the spectral Barron space (i.e., $q=1$),
as well as for the Hilbert-Sobolev space (i.e., $q=2$) as special (extreme) cases
(see \cite{Barron93UniversalApproximationBounds,Siegel20ApproximationRatesNeural}).
To the best of our knowledge, our paper is the first
work that treats the space-time generalization and the different types of integrability.
Our main result (again with simplifications to address the setting of time-dependent \glspl{pde}; the full generality can be found in \cref{thm:approximation_sobolev_space})
can be stated as follows:
\begin{theorem}
\label{thm:mainIntro}
    For $i\in\{1,2\}$ let $n_i \in\nn{}$,
    and $1\leq q_i\leq 2\leq p_i< \infty$ with $p_1\leq p_2$.
    Let $\vartheta(t_1,t_2)\gtrsim \eabs{t_1}^{\gamma_1}\eabs{t_2}^{\gamma_2} $
    for some $\gamma_1, \gamma_2>1$ and any $(t_1, t_2)\in \rr{2}$,
    and
    $\omega(t, x) \gtrsim \eabs{t}^{n_1}\eabs{x}^{n_2}$
    for any $(t,x)\in \rr{d+1}$
    moreover let
    \begin{align*}
        \tilde{\omega}(t,x):=
            \omega(t,x)
            \eabs{t}^{2(1-\frac{1}{q_1})+1}
            \eabs{x}^{(d+1)(1-\frac{1}{q_2})+1}.
    \end{align*}

    For an activation function
    $\sigma\in W_{n_1,\infty}^{n_2,\infty}(\vartheta;\rr{}, \rr{}){\setminus}\{0\}$
    and $u\in \mathscr{F}L^{q_1,q_2}(\tilde{\omega}; \rr{},\rr{d})$ there exists a constant $C>0$ such that
    \begin{align}
    \label{eq:ApproximationBoundIntro}
        \inf_{u_{N}\in\Sigma_{1,d}^{N}(\sigma)}
        \norm{u-u_N}{W_{n_1,p_1}^{n_2,p_2}([0,T],[-1,1]^d)}
        &\leq
        C N^{-\frac{1}{2}} T^{\frac{1}{p_1}} 2^{\frac{d}{p_2}}
        \norm{u}{\FL{q_1,q_2}(\tilde{\omega};\rr{}, \rr{d})}.
    \end{align}
\end{theorem}

It is worth to mention that our target class includes the setting of
\cite{Parhi23ModulationSpacesCurse},
where the authors measured the accuracy
of approximating a given function in the weighted
Feichtinger's Segal algebra
(i.e., the weighted modulation space with $p=q=1$ where the weight $\omega$ is the Bessel potential).
The relation between anisotropic weighted Fourier-Lebesgue spaces
and weighted modulation spaces is presented in \cref{rem:modulation_space}.

Note that our analysis also covers the case of a single block.
To see this, we consider the case with $u(t,x)=\ch{[0,1]}{t}u(x)$, $T=1$, and $p_1=q_1=2$.
Then, every derivative with respect to $t$ vanishes,
thereby reducing the Sobolev norm on the left side to the second variable.
On the right side, we see with the product structure of $u$,
the norm of the product can be split into the product of the norms.
For $p_1=2$ we can apply Parseval's theorem and see that every factor that is related to $t$ is trivially equal to $1$.
For the application to time-dependent \glspl{pde},
this means that we define a new space variable
$X := (t, x)= (t, x_1, \cdots , x_d)\in \rr{d+1}$
and instead of considering different integrability-
and differentiability degrees we have to choose
$p_1=p_2=p$, $q_1=q_2=q$
and $n= \min\{n_1,n_2\}$.
The single-block version of \cref{thm:mainIntro} is given by
\begin{theorem}
    \label{thm:ReducedMainIntro}
    Let $d, m, n \in\nn{}$, such that \(m\geq n\),
    and $1\leq q,\leq 2\leq p< \infty$.
    Let $ \vartheta(t) \gtrsim \eabs{t}^{\gamma}$
    for some $\gamma>1$
    and any $t \in \rr{}$,
    and
    $\omega(x) \gtrsim \eabs{x}^{n}$ for any $x\in \rr{d}$
    moreover let
    \begin{align*}
        \tilde{\omega}(x):=
            \omega(x)
            \eabs{x}^{(d+1)(1-\frac{1}{q})+1}.
    \end{align*}
    For an activation function $\sigma\in W^{m,\infty}(\vartheta;\rr{}){\setminus}\{0\}$
    and $u\in \mathscr{F}L^{q}(\tilde{\omega})$,
    there exists a constant $C>0$
    such that we have
    \begin{align}
    \label{eq:ReducedApproximationBoundIntro}
        \inf_{u_N\in\Sigma_{\sigma}^{N} }
        \norm{u-u_N}{W^{n,p}([0,T]\times[-1,1]^d)}
        &\leq
        C N^{-\frac{1}{2}} T^{\frac{1}{p}}2^{\frac{d}{p}}
        \norm{u}{\FL{q}(\tilde{\omega};\rr{d+1})}.
    \end{align}
\end{theorem}
(See \cref{sec1} for more details regarding the notations.)

We have seen that both versions (single-block and two-block)
of our approximation bound depends on some constant $C$.
In both cases the constant depends on the supremum of the domain,
the polynomial-decay exponent of the activation function,
the integral degree of the approximation error,
and the number of derivatives in the approximation error (thereby implicitely on the input dimension).
Unlike the recent literature (cf. \cite{Ma22UniformApproximationRates}) we also perform a theoretical analysis of the constants embedded within the approximation inequality.
With that we find scenarios, in which \glspl{snn} successfully overcome the curse of dimensionality.
The details on that can be found in \cref{prop:breaking_curse_of_dim}.

At the present moment our analysis
is limited to measuring the error
in the Bochner-Sobolev norm with $p\geq 2$.
This is because the currently available
techniques and methods rely on the fact
that the error is measured in a type-$2$ Banache space.
For $p<2$, the Bochner-Sobolev space
will be a type-$p$ Banach space and therefore,
the existing theory does not apply anymore.
We leave the investigation of this line of research for future projects.

The paper is organised as follows: In \cref{sec1} we briefly
recall some essential tools from harmonic and functional analysis.
Moreover we define the Bochner-Sobolev spaces considered in out setting.
Then we review the variation space and its connection
to conclude an approximation rate.
While our results regarding the
inclusion of Fourier-Lebesgue spaces in Bochner-Sobolev Spaces
for the high-order and low-order cases 
developed in \cref{subsec:ConvergenceInBochnerSobolevNorms,subsec:ConvergenceInLowOrderBochnerSobolevNorms}, 
respectively.
Our main contribution concerning the efficiency of
\glspl{snn} in approximating Fourier-Lebesgue functions
(in particular Barron functions)
with respect to Bochner-Sobolev norm can be found in 
\cref{sec:approximation_of_FL}
Finally in \cref{sec:experiments}
we present examples and experiments studies
to demonstrate the practical relevance
our finding.

\subsection{Notation}
Throughout this work, we denote the Schwartz space
of rapidly decreasing functions
on $\rr{d}$ by $\mathscr{S}(\rr{d})$.
The Lebesgue measure of a set $E$ is denoted by $|E|$ and its characteristic function by $\chi_E(x)$.
Binary relations on multi-indices act element-wise, i.e., for
$\alpha=(\alpha_1,\dots, \alpha_d),\beta =(\beta_1,\dots, \beta_d)\in \zzp{d}$
we say
\(\alpha \leq \beta\)
if and only if 
\(\alpha_i\leq \beta_i\)
for all 
\(i \in \{1, \dots, d\}
\)
and for $\beta\leq\alpha$ we define the difference 
\(\alpha-\beta\) as the multi-index
\(\alpha_1-\beta_1,\dots, \alpha_d-\beta_d\).
The magnitude of a multindex $\alpha\in\zzp{d}$ is given by
\(|\alpha| = \alpha_1+\dots+\alpha_d\),
and for any \(x\in \rr{d}\)
the element-wise exponent is
\(
x^\alpha = x_1^{\alpha_1}x_2^{\alpha_2}\dots x_d^{\alpha_d}
\)
and similarly, the partial derivatives are given by
\(\partial ^\alpha = \partial_{x_1}^{\alpha_1}\partial_{x_2}^{\alpha_2}
\dots \partial_{x_d}^{\alpha_d}\).
For a scalar $x$ we denote the ReLU function as $(x)_+:=\max\{0,x\}$ for $x\in\rr{d}$ we denote the Bessel potential via the bracket $\eabs{\cdot} := (1+|x|^2)^{\frac{1}{2}} $.

We stress that the constants that appear in inequalities may differ from
line to line; i.e., $C$ is a placeholder
for constants whose dependence is listed in the subscript or in the accompanying text.

%%%%%%%%%%%%%%%%%%%%%%%%%%%%%%%%%%%%
\section{Preliminaries}\label{sec1}
%%%%%%%%%%%%%%%%%%%%%%%%%%%%%%%%%%%%

As a first step towards our approximation results for \glspl{snn},
we recall some basic concepts and results.

\subsection{Anisotropic Weighted Fourier-Lebesgue Spaces}
As a first step, we will cover the basics of harmonic analysis
and introduce our concept of anisotropic weighted Fourier-Lebesgue spaces.

For the Fourier transform $\mathscr F$ and its inverse $\mathscr {F}^{-1}$ we use the convention with symmetric normalization,
i.e.,
\begin{align*}
(\mathscr Ff)(\xi )
&\equiv
\frac{1}{(2\pi)^{\frac{d}{2}}}\int _{\rr{d}} f(x)e^{-i\scal  x\xi }\, dx
\qquad\text{and}\qquad
(\mathscr {F}^{-1}\hat{f})(x )
\equiv
\frac{1}{(2\pi)^{\frac{d}{2}}}\int _{\rr{d}} \hat{f}(\xi)e^{i\scal  x\xi }\, d\xi,
\end{align*}
and we will write $\hat{f}$ as a short form of $\mathscr{F}f$.
Both integrals are well defined when $f\in L^1(\rr d)$ and $\hat{f}\in L^1(\rr d)$, respectively.

For the Fourier transform on multiple blocks,
we first let $\mathscr {F}_1$ and $ \mathscr {F}_2$
denote the (partial) Fourier transform
with respect to the first and second block of variables, respectively.
That is,
for a function $f\in L^{1}(\rr{d_1}\times\rr{d_2})$
we have
\begin{align*}
(\mathscr {F}_1f)(\xi,y)&
\equiv \frac{1}{(2\pi)^{\frac{d_1}{2}}}
    \int _{\rr{d_1}} f(x, y)e^{-i\scal  x\xi }\, dx,
\qquad\text{and}\\
(\mathscr {F}_2f)(x,\eta)
&\equiv \frac{1}{(2\pi)^{\frac{d_2}{2}}}
    \int _{\rr{d_2}} f(x, y)e^{-i\scal  y\eta }\, dy.
\end{align*}
By combining the partial Fourier transforms,
we get the two-block Fourier transform
\begin{equation}\label{eq:spacetimeFourierTransfrom}
\mathscr{F}f(\xi,\eta)
= (\mathscr {F}_2(\mathscr {F}_1f))(\xi,\eta)
= (\mathscr {F}_1(\mathscr {F}_2f))(\xi,\eta)
\end{equation}
for $f\in L^{1}(\rr{d_1}\times\rr{d_2})$.
A similar notion to this is the
so-called \emph{spacetime Fourier transform}
which is commonly used in the analysis of dispersive partial differential equations
\cite{Tao06NonlinearDispersiveEquations}.
It is formally defined for $u: \rr{}\times\rr{d}\rightarrow \cc{}$ as
\begin{align*}
(\mathscr{F}{u})(\tau, \xi):=\int_{\mathbf{R}} \int_{\mathbf{R}^d} u(t, x) e^{-i(t \tau+x \cdot \xi)} d t d x
\end{align*}
and can thereby be seen as a two-block Fourier transform with $d_1=1$.

For the weights in our norm, we will consider weight functions
$\omega : \rr{d_1}\times\rr{d_2} \rightarrow (0, \infty)$
which are measurable and such that $\omega(x,y), 1/\omega(x,y)>0$
for any $(x,y)\in \rr{d_1}\times\rr{d_2}$.
Later in our embedding result
we require that the growth of the weight is bounded in some way.
For that we adopt the notion of moderateness from \cite{Pilipovic11MicroLocalAnalysisFourier}, i.e.,
for two weight functions $\omega$ and $v$
on $\rr{d_1}\times\rr{d_2}$,
we say $\omega$ is $v$-moderate if
\begin{equation}\label{moderate}
	\omega (x_1+x_2, y_1+y_2) \leq C\omega (x_1,y_1)v(x_2,y_2)
\end{equation}
for some uniform constant $C$.
If $v$
in \eqref{moderate} can be chosen as a polynomial (or exponential),
then $\omega$ is
called polynomially
(or exponentially) moderated.
As an example, the weight
$\omega (x, y )=\eabs x ^s\eabs y ^\sigma$,
with $s, t\in \rr{}$,
is polynomially moderated
and for $r,\rho, s, t> 0$
the weight $\omega(x,y ) = e^{r|x|^s + \rho|y|^t}$
is exponentially moderated.
In the usual convention of Harmonic analysis (see \cite[Page 4]{Hormander05AnalysisLinearPartial}),
we will also assume that weights are submultiplicative, i.e., a weight $\omega$ is $v$-moderate with $v\equiv\omega$.
Contrary to the upper bound provided by moderatedness,
we additionally require that the weight is lower bounded in the following way:
For two weight functions $\omega, \vartheta$
we say that $\omega$ is elliptic
with respect to $\vartheta$
if 
\begin{equation}\label{eq:ellipticity_cond}
0<\vartheta(x, y)\leq c\,\omega(x,y)
\text{ for any }(x,y) \in \rr{d_1}\times\rr{d_2}
\end{equation}
with a uniform constant $c>0$.

In \cite{Benedek61SpaceMixedNorm} Benedek and Panzone
introduced the mixed (anisotropic) Lebesgue spaces
$L^{ \overrightarrow{p}}(\rr{d})$ 
with $\overrightarrow{p} = (p_1, \dots, p_d)\in [1,\infty]^d$.
A function $u$ is said to belong to $L^{ \overrightarrow{p}}(\rr{d})$
if
\begin{align*}
\norm{u}{L^{ \overrightarrow{p}}(\rr{d})}
\equiv
\left(
\int_{\rr{}}\dots
    \left(
        \int_{\rr{}}|u(x_1,x_2,\dots,x_d)|^{p_1} dx_1
    \right)^{{p_2}/{p_1}}
\dots dx_d
\right)^{{1}/{p_d}}<\infty
\end{align*}
with obvious interpretation when $p_i = \infty$.
The anisotropic Lebesgue spaces
$L^{ \overrightarrow{p}}(\rr{d})$ are a generalization of 
Lebesgue spaces $L^{p}(\rr{d})$,
such that the integrability exponent is different for each variable.
In our case, we consider that case that the
exponent is constant within each block of variables,
but it is allowed to be different between the two blocks.
For that we let $d_1,d_2 \in \nn{}$, $p, q\in [1,\infty]$, and $U\subseteq\rrI$, $V\subseteq\rrII$
and define the norm
\begin{equation}\label{eq:mixed_norm_definition}
\norm{u}{L^{p, q}(U,V)}
\equiv
\left(\int_{U}
\left(\int_{V}\abs{u(x, y)}^q d y\right)^{p / q}
    d x\right)^{1 / p}.
\end{equation}
If $p=q$ this simplifies to $ L^{p}(U\times V)$
and in case $U=\rrI$ and $V=\rrII$ we use the short notation
$\norm{\cdot}{L^{p, q}(\rr{d_1},\rr{d_2})}=\norm{\cdot}{L^{p, q}}$
(not to be confused with the Lorentz spaces)
and further $\norm{\cdot}{L^{p, p}}=\norm{\cdot}{L^{p}}$.

A similar definition for the mixed Lebesgue norms
is used in the analysis of dispersive \glspl{pde} 
\cite{Tao06NonlinearDispersiveEquations} (see also \cite{Sogge13LecturesNonlinearWave}) where $U=I$ is an interval and $V=\rr{d}$.
We stress that the mixed Lebesgue norms
are the right spaces to deliver Strichartz estimates
which are essential to the well-posedness of certain \glspl{pde},
e.g., non-linear Schrödinger equations
or linear Schrödinger equations with time-dependent potentials \cite{Keel98EndpointStrichartzEstimates,Ginibre95GeneralizedStrichartzInequalities}.

Combining the concepts of Fourier-Lebesgue spaces, two-block Fourier transform, weight functions, and anisotropic Lebesgue spaces leads to the definition of anisotropic weighted Fourier-Lebesgue spaces:
\begin{definition}[Anisotropic Weighted Fourier-Lebesgue Spaces]
\label{def:weighted_FL_mixed}
Let $p, q\in [1,\infty ]$ and $\omega$
be a weight defined over $\rrIxII$.
The anisotropic (weighted) Fourier-Lebesgue space
$\mathscr{F}L^{p, q}(\omega , \rrI,\rrII)$ 
consists of all
$f\in L^1(\rr{d_1} \times \rr{d_2})$ 
such that
\begin{equation}\label{FLnorm}
\begin{aligned}
	\norm{f}{\FL{p, q}(\omega;\rrI,\rrII)}
    &\equiv \norm{\omega\mathscr{F} f}{L^{p, q}(\rr{d_1}, \rr{d_2} )}
\end{aligned}
\end{equation}
is finite.    
\end{definition}

Here and in what follows we use the notation
$\mathscr F L^{p, q}(\omega)$ instead of
$\mathscr F L^{p, q}(\omega; \rrI,\rrII)$.
If $\omega =1$, then the notation $\mathscr FL^{p, q}$
is used instead of $\mathscr FL^{p, q}(1 )$.
We note that if $d_1=d_2 = d$,
$\omega (\xi , \eta)=\eabs {(\xi, \eta)} ^s$,
then
$\mathscr{F}L^{p, p}(\omega )$ is the Fourier image of the Bessel potential space
$H^p(\omega ; \rr{2d})$
(cf. \cite{Bergh76InterpolationSpacesIntroduction}).
Furthermore, if $p=q$ we write $\mathscr F L^{p}(\omega)$
instead of $\mathscr F L^{p, p}(\omega)$.

\begin{remark}
\label{rem:modulation_space}
    We note that if $d_1=d_2=d$, $\omega$ is a weight function on $\rr{d}\times\rr{d}$
    and $u, \varphi\in \mathscr{S}(\rr{d})$,
    then for functions of the form
    \[
        f(x,y) = \mathscr{F}_1^{-1}(u(y )\varphi(y-\,\cdot)),
    \]
    we get
    $$
    \mathscr{F} f(\xi,\eta) = \mathscr{F}_2(u(\cdot)\varphi(\cdot-\xi)(\eta)
    =
    (2\pi )^{-{d}/2}
        \int _{\rr{d}} u(y)\varphi(y-\xi)e^{-i\scal  y\eta }\, dy
        \equiv V_\varphi u(\xi, \eta),
    $$
    where $V_\varphi u$ is the \emph{short-time Fourier transform}
    of the signal $u$ with respect to a window function $\varphi $.
    Thus, when restricting to functions of this type, then our definition of Fourier-Lebesgue spaces $\mathscr{F}L^{p,q}(\omega)$
    agrees with the definition of modulation spaces. 
    To see that, we recall that for $p,q\in [1,\infty]$,
    $\varphi\in \mathscr{S}(\rr{d})$,
    the weighted modulation space
    $M^{p, q}(\omega)\left(\mathbf{R}^d\right)$ consists of all
    $u \in \mathscr{S}^{\prime}\left(\mathbf{R}^d\right)$
    such that
    $$
    \left(\int_{\rr{d}}\left(\int_{\rr{d}}
        \left|\omega(x, \xi)V_{\varphi}u(x,\xi)\right|^q
            d x\right)^{p / q}
            d \xi\right)^{1 / p}<\infty,
    $$
    (with obvious modification when $p=\infty$ or $q=\infty$ ).
    More details on modulation spaces can be found in 
    \cite{Grochenig01FoundationsTimeFrequencyAnalysis}.
\end{remark}

\subsection{Bochner-Sobolev space}\label{sec:Bochner_Sobolev_space}
In our main approximation result, we will measure the approximation error in terms of a mixed-degree Sobolev space.
In order to provide a profound definition for that,
we now review some properties of \emph{Bochner-Sobolev spaces}.
We start by recalling the definition of Sobolev space.

\begin{definition}[Sobolev space]
Assume that $\Omega$ is an open subset of $\rr{d}$, and
let $n \in \zzp{}$, $1 \leq q \leq \infty$. 
The Sobolev space $W^{n, q}(\Omega)$
consists of functions $u \in L^{q}(\Omega)$
such that for every multi-index $\alpha$ with $|\alpha| \leqslant n$
the partial derivative $\partial^{\alpha} u$ exists
and $\partial^{\alpha} u \in L^{q}(\Omega)$.
Thus
\begin{align*}
W^{n, q}(\Omega):=\left\{f \in L^{q}(\Omega):
\partial^{\alpha} f \in L^{q}(\Omega) \text { for all }
\alpha \in \zzp{d} \text { with }|\alpha| \leq n\right\}.
\end{align*}
Furthermore, for $f \in W^{n, q}(\Omega)$ and $1 \leq q<\infty,$ we define the norm
\begin{align*}
\norm{f}{W^{n, q}(\Omega)}:=\left(\sum_{0 \leq|\alpha| \leq n}
    \norm{\partial^{\alpha} f}{L^{q}(\Omega)}^{q}\right)^{1 / q}
\end{align*}
and
\begin{align*}
\norm{f}{W^{n, \infty}(\Omega)}:=\max _{0 \leq|\alpha| \leq n}
    \norm{\partial^{\alpha} f}{L^{\infty}(\Omega)}.
\end{align*}
\end{definition}

We consider the so-called
\emph{Bochner space}
which is the natural generalisation of Lebesgue integral
to the case that the function has values
in an arbitrary Banach space.
More details about Bochner spaces can be found in
e.g., \cite[Section 3]{Kainen13BochnerIntegralsNeural},
\cite[Chapter 10]{Light85ApproximationTheoryTensor}.
In this work, we consider a much simpler situation
that we fix the Banach space to be the Sobolev space
in order to avoid technicalities
and focus more on the targeted result.
Mainly,
we deal with functions which belong to some Sobolev space
on $\rr{d_1}$ (i.e., $W^{m,p}(\rr{d_1})$)
with values in another Sobolev space on $\rr{d_2}$
(i.e., $W^{n,q}(\rr{d_2})$).

Note that the extension of our results to
the general case of Bochner space
can be thought as a future work.

\begin{definition}[Bochner-Sobolev space]
Let $1 \leq p,q \leq \infty$, $m, n\in \zzp{}$,
$U \subseteq \rr{d_1}$,
and  $V\subseteq \rr{d_2}$. Let
$W_{m,p}^{n, q}(U, V)$
be defined as follows 
$$
W_{m,p}^{n, q}(U, V)
:=\left\{f \in L^{p}
\left(U, W^{n, q}(V)\right):
\partial _x^ \alpha f \in L^{p}\left(U, W^{n, q}(V)\right)
 \text { for all }  |\alpha| \leq m\right\}
$$
such that
\begin{equation}\label{eq:BochnerSobolevNorm}
    \Vert f\Vert _{W_{m,p}^{n, q}(U, V)}
    :=
    \bigg(\sum_{|\alpha| \leq m}
    \Vert \partial _x^ \alpha  f\Vert_{L^{p}\left(U, W^{n, q}(V)\right)}^p\bigg)^{1/p}
        <\infty,
\end{equation}
when $1\leq p, q< \infty$,
with the obvious modifications when $p=\infty$ and/or $q=\infty$.
\end{definition}
More precisely, \eqref{eq:BochnerSobolevNorm} is the same as
\begin{equation}\label{eq:BochnerSobolevNorm_as_Lebesgue}
    \Vert f\Vert _{W_{m,p}^{n, q}(U, V)}
    :=
    \bigg(\sum_{|\alpha| \leq m}
    \bigg\Vert \bigg(\sum_{|\beta| \leq n}
    \Vert \partial_y^\beta \partial_x^\alpha f
    \Vert_{L^{q}\left(V\right)}^q\bigg)^{1/q}
    \bigg\Vert_{L^{p}\left(U\right)}^p \bigg)^{1/p}.
\end{equation}

Note that if $n=m=0$, then
$W_{0,p}^{0,q}(U, V) = L^p(U, L^q(V))$.
Hence, we shall write
$L^{p,q}(U, V) := W_{0,p}^{0,q}(U, V)$,
where $L^{p,q}(U, V)$
stands for the Lebesgue integral with respect to 
$L^q(V)$ then $L^{p}(U)$.

%%%%%%%%%%%%%%%%%%%%%%%%%%%%%%%%%%%%%%%%%%%%%%%%%%%%%%%%
\subsection{Smoothing by convolution}\label{sec:SmoothConv}
%%%%%%%%%%%%%%%%%%%%%%%%%%%%%%%%%%%%%%%%%%%%%%%%%%%%%%%%

In the proof of \cref{lem:smoothness_lemma_high_degree}
we require $L^1$ integrability of the
Fourier transform characteristic functions,
which is not given a-priory in general.
In order to get around this limitation,
we first approximate the function 
by a smoothed version via convolution
and then we take the limit such that
the convolution is with the dirac-delta distribution.
For that we consider the following theory:

Following \cite[Lemma 2.4 and Definition 3.1]{Tartar07IntroductionSobolevSpaces} 
for $\epsilon > 0$ we define the smoothing sequence
\begin{align*}
    \rho_\epsilon(x):=
    \frac{1}{\epsilon^d\|\phi\|_{L^1}}\phi
    \left(\frac{x}{\epsilon}\right)\quad
    \text{with}
    \quad\phi(x)=\exp\left(-\frac{1}{1-|x|^2}\right)\chi_{B_1(0)}(x),
\end{align*}
where $B_1(0)$ is the closed unit ball.
Thus, \(\rho_\epsilon\in C_c^\infty\) with 
\begin{align}
\label{eq:mollifier_norms}
    \|\rho_\epsilon\|_{L^1}=1
    \quad\text{and}\quad
    \|\rho_\epsilon\|_{L^2}
    =\left\|\frac{1}{\epsilon^d}\rho_1
    \left(\frac{\cdot}{\epsilon}\right)\right\|_{L^2}
    =\frac{1}{\epsilon^{d/2}}\|\rho_1^2\|_{L^1}^{1/2}
    \leq\frac{1}{\epsilon^{d/2}}\|\rho_1\|_{L^1}^{1/2}
    =\frac{1}{\epsilon^{d/2}}
\end{align}
by substitution in multiple variables and using \(\rho_1(x)<1\) for all \(x\in\rr{d}\).
For a domain $\Omega\subset \rr{d}$ we define the smoothed characteristic function of \(\Omega\) as
\begin{align*}
    \chi_{\Omega}^\epsilon:=\chi_{\Omega} \conv \rho_{\epsilon}.
\end{align*}
With \(\Omega_\epsilon := \left\{x\in \rr{d}\mid B_\epsilon(x)\cap\Omega\neq\emptyset\right\}\)
being the domain extended by a margin of width $\epsilon$
and
\(\Omega_{-\epsilon}:= \left\{x\in \rr{d} \mid B_\epsilon(x)\subset\Omega\right\}\)
being the domain shrinked by a margin of width $\epsilon$
we see that the smoothed characteristic funciton is smoothly decaying over the margin
\(\Omega_{\epsilon}\setminus \Omega_{-\epsilon}\)
of width \(2\epsilon\).

For bounded \(\Omega\), we see $\chi_{\Omega}^\epsilon\in C^\infty$
by \cite[Lemma 2.3]{Tartar07IntroductionSobolevSpaces},
it has bounded support by \cite[Equation (2.2)]{Tartar07IntroductionSobolevSpaces}
as \(\Omega\) is bounded and the support of \(\rho_\epsilon\) is bounded,
and $\chEF{\Omega}{\epsilon}\to_{\epsilon\to 0}\chF{\Omega}$ by \cite[Lemma 3.2]{Tartar07IntroductionSobolevSpaces}.
Thus, \(\chi_{\Omega}^\epsilon\in L^1\) and
\begin{align*}
    \mathscr{F}\left(\chi_{\Omega}^\epsilon\right)
    =\mathscr{F}\left(\chi_{\Omega}\conv\rho_\epsilon\right)
    =\mathscr{F}\left(\chi_{\Omega}\right)\mathscr{F}\left(\rho_\epsilon\right).
\end{align*}
This yields the upper bound on the $\FL{1}$-norm
\begin{align}
\label{eq:smooth_char_L1_bound}
\begin{split}
    \norm{\mathscr{F}\left(\chi_{\Omega}^\epsilon\right)}{L^1(\rr{d})}
    &\leq\norm{\mathscr{F}\left(\chF{\Omega}\right)}{L^2(\rr{d})}\norm{\mathscr{F}\left(\rho_\epsilon\right)}{L^2(\rr{d})}
    =\norm{\chF{\Omega}}{L^2(\rr{d}}\norm{\rho_\epsilon}{L^2(\rr{d})}\\
    &=|\Omega|^{1/2}\norm{\rho_\epsilon}{L^2(\rr{d})}<\infty
\end{split}
\end{align}
by Hölder's inequality.
For \(1\leq s\leq \infty\) we further get the upper bound on the $\FL{s}$-norm
\begin{align}
\label{eq:smooth_char_Lp_bound}
\begin{split}
    \norm{\mathscr{F}\left(\chi_{\Omega}^\epsilon\right)}{L^s(\rr{d})}
    &\leq\norm{\mathscr{F}\left(\chF{\Omega}\right)}{L^s(\rr{d})}\norm{\mathscr{F}\left(\rho_\epsilon\right)}{L^\infty(\rr{d})}\\
    &\leq \norm{\mathscr{F}\left(\chF{\Omega}\right)}{L^s(\rr{d})}\norm{\rho_\epsilon}{L^1(\rr{d})}
    =\norm{\mathscr{F}\left(\chF{\Omega}\right)}{L^s(\rr{d})}
\end{split}
\end{align}
by taking the supremum over \(\mathscr{F}\left(\rho_\epsilon\right)\) and using the $L^1$-norm as upper bound on the Fourier Transform
(see \cite[page 294]{Jones01LebesgueIntegrationEuclidean}).
Note, however, that \(\mathscr{F}\left(\chF{\Omega}\right)\) is not necessarily in \(L^s(\rr{d})\) for \(s<2\), which will be discussed in \cref{sec:structure_of_domain}.
Nevertheless, \cref{eq:smooth_char_L1_bound} shows that \(\mathscr{F}\left(\chi_{\Omega}^\epsilon\right)\in L^1\) for all \(\epsilon>0\).

\subsection{Variation Space and Approximation Rate}
\label{sec:prelim:VariationSpace}
In this section we provide a brief overview over tools
that are necessary to establish our approximation bound in \cref{sec:approximation_of_FL}.
Namely, we will review variation spaces and the resulting approximation rates
for these spaces,
all of that already with the focus on shallow neural networks with two blocks of variables.
The extension to
multiple blocks is straightforward,
furthermore, the one block of variables is
a particular special case of our analysis.
For a general overview over variation spaces we refer the interested reader to
\cite{DeVore98NonlinearApproximation,Kurkova01BoundsRatesVariableBasis,Siegel22SharpBoundsApproximation,Siegel23CharacterizationVariationSpaces}.

The definition of a variation space is based on some dictionary that is a subset of some function-space.
In our work, we mostly rely on the theory of \cite{Siegel22SharpBoundsApproximation},
which means that we define a dictionary as $\mathbb{D}\subseteq\mathcal{B}$ with some Banach space $\mathcal{B}$ (other works such as \cite{Siegel23CharacterizationVariationSpaces} considers the stronger assumption that the dictionary has to be in some Hilbert space).
In order to approximate some target functions via the dictionary,
we will consider the linear combinations of $N\in\nn{}$ elements of this dictionary
and bound the weight of the combination with respect to $\ell_1$ by some constant $M>0$.
That is, we consider the set
\begin{align*}
    \Sigma_{N,M}(\mathbb{D}):=\left\{\sum_{j=1}^Na_jh_j:
        h_j\in\mathbb{D},\sum_{j=1}^N\abs{a_j}\leq M\right\}.
\end{align*}
The variation space $\mathcal{K}(\mathbb{D})$
associated with this dictionary can roughly be seen as
the set of functions that can be realized via an infinite linear combination.
It is defined via the variation norm as follows:
\begin{definition}
    Let $\mathcal{B}$ be a Banach space and $\mathbb{D}\subseteq\mathcal{B}$ be a dictionary.
    Then for $f\in\mathcal{B}$, the variation norm of $\mathbb{D}$ is defined as
    \begin{align*}
        \norm{f}{\mathcal{K}(\mathbb{D})}&:=\inf\{c>0:f/c\in\overline{\operatorname{conv}(\pm\mathbb{D})}\}
    \intertext{were, $\overline{\operatorname{conv}(\pm\mathbb{D})}$
        is the closure of the convex hull of $\mathbb{D}\cup(-\mathbb{D})$.
        The corresponding variation space is then the set of functions with finite variation norm}
        \mathcal{K}(\mathbb{D})&:=\{f\in\mathcal{B}:\norm{f}{\mathcal{K}(\mathbb{D})}<\infty\}.
    \end{align*}
\end{definition}

The main focus of this work is on shallow networks with two blocks of variables.
That is, we consider an activation function $\sigma:\rr{2}\to\rr{}$ and the
parameter spaces $\Lambda=\Lambda_\xi\times\Lambda_b$,
where $\Lambda_\xi$ is the set of admissible input weights
and $\Lambda_b$ is the set of admissible input biases.
Every single neuron in the neural network is
then composed of a parametrized affine function
\begin{align*}
    T(\cdot,\cdot\cdot;\xi,b):\rrIxII\to\rr{2}
    \qquad\text{with}\qquad
    T(x,y;\xi,b):=(\xi_1\cdot x+b_1,\xi_2\cdot y+b_2)
\end{align*}
with the weight $\xi=(\xi_1,\xi_2)\in\Lambda_\xi$ and bias $b=(b_1,b_2)\in\rr{2}$
and the activation function $\sigma:\rr{2}\to\rr{}$.
In this setting, we consider the dictionary
\begin{align*}
    \mathbb{D}_\sigma^{d_1,d_2}:=\{\sigma(T(\cdot,\cdot\cdot;\xi,b))\in \mathcal{B}:(\xi_1,\xi_2)\in\Lambda_\xi,(b_1,b_2)\in\Lambda_b\}.
\end{align*}
Throughout our work,
we also consider the dictionary with a scaled activation function.
That is, for some strictly positive function
$\phi:\Lambda_\xi\times\Lambda_b\to(0,\infty)$
we consider
\begin{align*}
    \mathbb{D}_{\phi,\sigma}^{d_1,d_2}:=\{\phi(\xi,b)\sigma(T(\cdot,\cdot\cdot,\xi,b))\in \mathcal{B}:(\xi_1,\xi_2)\in\Lambda_\xi,(b_1,b_2)\in\Lambda_b\}.
\end{align*}

A central element of our main approximation result is an upper bound on the approximation rate for Banach spaces of Rademacher type $2$, which is typically attributed to Maurey's (see \cite{Pisier80RemarquesResultatNon,Barron93UniversalApproximationBounds,Siegel22SharpBoundsApproximation}.
Before stating Maurey's result,
we introduce the definition for the general Rademacher type $p$:
\begin{definition}[Rademacher Type $p$ {\cite[Definition 6.2.10]{Albiac06TopicsBanachSpace}}]
    A Banach Space $\mathcal{B}$ is said to have Rademacher type $p$ for some $1\leq p\leq 2$ if there is a constant $C_{p,\mathcal{B}}$ such that for every finite set of vectors $\{x_i\}_{i=1}^n$ in $\mathcal{B}$,
    \begin{align*}
        \left(\mathbb{E}\left\{
                \normbig{\sum_{i=1}^n\varepsilon_ix_i}{\mathcal{B}}^p
            \right\}\right)^{1/p}
        \leq
        C_{p,\mathcal{B}}\left(\sum_{i=1}^n\normbig{x_i}{\mathcal{B}}^p\right)^{1/p}.
    \end{align*}
    In short, we say that $\mathcal{B}$ is a type-$p$ Banach space.
\end{definition}

A version of Maurey's approximation result
that is already tailored to functions
in the variation space of a dictionary is stated in \cite{Siegel22SharpBoundsApproximation} as follows:
\begin{proposition}[Approximation Rate in Type-2 Banach Spaces]
\label{prop:approximation_type2}
    Let $\mathcal{B}$ be a type-2 Banach space
    and $\mathbb{D} \subset \mathcal{B}$ be a dictionary
    with $K_{\mathbb{D}}:=\sup _{d \in \mathbb{D}}\|d\|_{\mathcal{B}}<\infty$.
    Then for $f \in \mathcal{K}(\mathbb{D})$, we have
    \begin{align*}
        \inf_{f_N \in \Sigma_{N, M_f}(\mathbb{D})}\left\|f-f_N\right\|_{\mathcal{B}}
        \leq 4 C_{2,\mathcal{B}} K_{\mathbb{D}}\norm{f}{\mathcal{K}(\mathbb{D})} N^{-\frac{1}{2}}
    \end{align*}
    with $M_f=\norm{f}{\mathcal{K}(\mathbb{D})}$.
\end{proposition}
A discussion regarding possible extensions of Maurey's result
can be found in 
\cite[Section 8]{DeVore98NonlinearApproximation}.

The prerequisite for the approximation result is that
the function $f$ is in the variation space $\mathcal{K}(\mathbb{D})$
of the dictionary $\mathbb{D}$.
A sufficient condition for this can be obtained by a slight modification of \cite[Lemma 3]{Siegel23CharacterizationVariationSpaces}.
Namely, we ask for boundedness instead of compactness,
consider a Banach space instead of a Hilbert space,
and treat only the implication in one direction.
The precise formulation of this modification is as follows:
\begin{proposition}[{\cite[Lemma 3]{Siegel23CharacterizationVariationSpaces}}]
\label{prop:variation_norm}
    Let $\mathcal{B}$ be a Banach space and suppose that $\mathbb{D}\subset\mathcal{B}$ is bounded.
    Then $f\in\mathcal{K}(\mathbb{D})$ if there
    exists a Borel measure $\mu$ on $\mathbb{D}$ such that
    \begin{align*}
        f=\int_{\mathbb{D}}i_{\mathbb{D}\to\mathcal{B}}\,d \mu.
    \end{align*}
    Moreover,
    \begin{align*}
        \norm{f}{\mathcal{K}(\mathbb{D})}
        =\inf\left\{\norm{\mu}{}:
        f=\int_\mathbb{D}i_{\mathbb{D}\to\mathcal{B}}\,d \mu\right\},
    \end{align*}
    where the infimum is taken over all Borel measures $\mu$
    defined on $\mathbb{D}$, and $\| \mu\|$ is the total variation\footnote{The interested reader can check \cite{Cohn13MeasureTheorySecond} for more details regarding measure theory.} of $\mu$. 
\end{proposition}

\section{Convergence Rates with Respect to Bochner-Sobolev Norms}
\label{sec:main}

In this section, we extend
the single-block approximation results
for functions in the Barron space
with Hilbert-Sobolev error measure
of \cite{Siegel20ApproximationRatesNeural}
towards our setting with
functions in the anisotropic weighted Fourier-Lebesgue spaces
and the Bochner-Sobolev norm as error measure.
The goal of these approximation results is to show
that functions that lie in the anisotropic weighted Fourier-Lebesgue spaces
can be approximated well by shallow neural networks.
This is done by providing an upper bound on the approximation error that is only dependent on the number of neurons, the domain of the function and on the degree of the chosen norms.

\subsection{Inclusion of Fourier-Lebesgue Spaces in High-Degree Bochner-Sobolev Spaces}
\label{subsec:ConvergenceInBochnerSobolevNorms}
The first step in the approximation result is to show that functions in the anisotropic weighted Fourier-Lebesgue spaces lie in the Bochner-Sobolev space.
\begin{lemma}[Inclusion in high-degree Bochner-Sobolev Space]
\label{lem:smoothness_lemma_high_degree}
    Let
    $1\leq s_i,t_i\leq 2\leq p_1\leq p_2\leq \infty$
     and $d_i, n_i \in\zzp{}$
    such that $\frac{1}{s_i} + \frac{1}{t_i} + \frac{1}{p_i} = 2$,
    for $i\in\{1,2\}$.
    Let $\omega(x,y)$ be a polynomially moderated weight function
    defined on $\rrIxII$, elliptic with respect to
    $\eabs{x}^{n_1}\eabs{y}^{n_2}$,
    and
    $U\subset \rrI$,
    $V\subset \rrII$ be bounded and measurable with non-empty interior.
    Let $f\in \FL{t_1,t_2}(\omega;\rrI,\rrII)$
	then we have
    \begin{align}
        \label{res:high_degree_bound}
		\norm{f}{W_{n_1,p_1}^{n_2, p_2}(U, V)}
        \leq 
            C_{n_1,n_2,p_1,p_2}\norm{{\chF{U}}}{\FL{s_1}(\rrI)}
            \norm{{\chF{V}}}{\FL{s_2}(\rrII)}
            \norm{f}{\FL{t_1,t_2}(\omega;\rrI,\rrII)},
	\end{align}
    where $C_{n_1,n_2,p_1,p_2}>0$ depends only on $n_1, n_2$, $p_1$, and $p_2$.   
\end{lemma}

Note that the ellipticity assumption on the weight $\omega$ results in $1/\omega\in L_1$ as
\begin{align*}
0\leq\frac 1{\omega(x,y)} \leq \frac 1{\eabs{x}^{n_1}\eabs{y}^{n_2}}
\in L^{1}(\rrIxII)
\end{align*}
whenever $n_1>d_1$ and $n_2>d_2$.

In the proof of this Lemma and later
in \cref{lem:smoothness_lemma_low_degree},
we use the following relation between the Sobolev-norm
and the anisotropic Lebesgue norm, which is due to the monotonicity
of \(\ell_p\)-norms (\(\|\cdot\|_{p_1},\|\cdot\|_{p_2}\leq\|\cdot\|_1\)
for \(p_1,p_2\geq 1\)) and the Minkowski inequality:
\begin{align}
    \label{eq:sobolev_to_lpq}
    \begin{split}
        \norm{f}{W_{m,p_1}^{n, p_2}(U, V)}
        &=\left(\sum_{|\alpha|\leq m}\norm{\norm{\partial_x^\alpha f(x,\cdot)}{W^{n,p_2(V)}}}{L^{p_1}(U)}^{p_1}\right)^{\frac{1}{p_1}}\\
        &=\left(\sum_{|\alpha|\leq m}\normbig{\left(\sum_{|\beta|\leq n}\norm{\partial_y^\beta\partial_x^\alpha f}{L^{p_2}(V)}^{p_2}\right)^{\frac{1}{p_2}}}{L^{p_1}(U)}^{p_1}\right)^{\frac{1}{p_1}}\\
        &\leq\sum_{|\alpha|\leq m}\sum_{|\beta|\leq n}\norm{\norm{\partial_y^\beta\partial_x^\alpha f(x,y)}{L^{p_2}(V)}}{L^{p_1}(U)}\\
        &=\sum_{|\alpha|\leq m}\sum_{|\beta|\leq n}\norm{\chF{U}\chF{V}\partial_y^\beta\partial_x^\alpha f}{L^{p_1,p_2}}.
    \end{split}
\end{align}

\begin{proof}
    For a polynomially moderated weight $\omega$ we limit to the Schwartz class
    i.e., let $f \in \mathscr{S}(\rrIxII)$.
    Then we have by \cref{eq:sobolev_to_lpq} and \cref{prop:smoothing_convergence}
\begin{align}
    \label{eq:sobolev_to_sequence}
    \begin{split}
        \norm{f}{W_{n_1,p_1}^{n_2, p_2}(U, V)}
        &\leq\sum_{|\alpha|\leq n_1}\sum_{|\beta|\leq n_2}\norm{\chF{U}\chF{V}\partial_y^\beta\partial_x^\alpha f}{L^{p_1,p_2}}
        \\
        &=\lim_{\epsilon_1\to 0}\lim_{\epsilon_2\to 0}\sum_{|\alpha|\leq n_1}\sum_{|\beta|\leq n_2}
        \norm{\chEF{U}{\epsilon_1}\chEF{V}{\epsilon_2}\partial_y^\beta\partial_x^\alpha f}{L^{p_1,p_2}}.
    \end{split}
\end{align}

    With $f \in \mathscr{S}(\rrIxII)$
    we have $\partial_x^\alpha\partial_y^\beta f\in L^1$
    and $\partial_x^\alpha\partial_y^\beta f\in \FL{1}$
    for any
    $\alpha\in  \zzp{d_1}, \beta \in \zzp{d_2}$.
    Further, for any \(\epsilon_1,\epsilon_2>0\)
    we have \(\chEF{U}{\epsilon_1}\chEF{V}{\epsilon_2}\in L^1(\rrIxII)\)
    and \(\chEF{U}{\epsilon_1}\chEF{V}{\epsilon_2}\in \FL{1}(\rrI,\rrII)\)
    by construction of the smoothed characteristic function.
    Thus, by \cite[Section 13.B, Page 316]{Jones01LebesgueIntegrationEuclidean},
    \(
        \chEF{U}{\epsilon_1}\chEF{V}{\epsilon_2}
        \partial_y^\beta\partial_x^\alpha f
        \in L^1(\rrIxII)
    \)
    and
    \begin{align}
    \label{eq:convolution_theorem}
        \mathscr{F}(\chEF{U}{\epsilon_1}\chEF{V}{\epsilon_2}\partial_y^\beta\partial_x^\alpha f)
        =\mathscr{F}(\chEF{U}{\epsilon_1}\chEF{V}{\epsilon_2})\conv\mathscr{F}(\partial_y^\beta\partial_x^\alpha f).
    \end{align}
    Young's convolution inequality then shows
    \begin{align*}
        \norm{\mathscr{F}(\chEF{U}{\epsilon_1}\chEF{V}{\epsilon_2}\partial_y^\beta\partial_x^\alpha f)}{L^1}
        \leq\norm{\chEF{U}{\epsilon_1}\chEF{V}{\epsilon_2}}{\FL{1}}\norm{\partial_y^\beta\partial_x^\alpha f}{\FL{1}}<\infty.
    \end{align*}
    Hence, the inverse Fourier transform $\mathscr{F}^{-1}$ exists and 
    \begin{align*}
        \|\chEF{U}{\epsilon_1}\chEF{V}{\epsilon_2}
        \partial_y^\beta\partial_x^\alpha f\|_{L^{p_1,p_2}}
        = 
        \|\mathscr {F}^{-1}\left[\mathscr {F}\left(\chEF{U}{\epsilon_1}\chEF{V}{\epsilon_2} {\partial_y^\beta\partial_x^\alpha f}\right)\right]\|
                _{L{p_1,p_2}}.
    \end{align*}
    Therefore, using the Hausdorff-Young inequality for mixed Lebesgue spaces \cite[Section 12]{Benedek61SpaceMixedNorm} for $2\leq p_1\leq p_2\leq \infty$,
    where $\frac{1}{p_i} + \frac{1}{q_i} = 1$ (thus, $1\leq q_2\leq q_1\leq 2$),
    there exists a constant $C_{p_1,p_2}>0$
    depending on $p_1$ and $p_2$ such that
    \begin{align}
    \label{eq:hausdorff_young}
        \norm{\mathscr {F}^{-1}\left[\mathscr {F}\left(\
        \chEF{U}{\epsilon_1}\chEF{V}{\epsilon_2}
        {\partial_y^\beta\partial_x^\alpha f}
        \right)\right]}{L^{p_1,p_2}}
        \leq C_{p_1,p_2}
        \norm{\mathscr{F}\left(\chEF{U}{\epsilon_1}\chEF{V}{\epsilon_2}
        {\partial_y^\beta\partial_x^\alpha f}\right)}{L^{q_1,q_2}}.
    \end{align}
    
    We have seen in \cref{eq:convolution_theorem} that the convolution theorem applies and with Young's convolution inequality we get
    \begin{align*}
        \norm{\mathscr{F}\left(\chEF{U}{\epsilon_1}\chEF{V}{\epsilon_2}
        {\partial_y^\beta\partial_x^\alpha f}\right)}{L^{q_1,q_2}}
        &=\norm{\mathscr{F}\left(\chEF{U}{\epsilon_1}\chEF{V}{\epsilon_2}\right)
        \conv
        \mathscr{F}\left({\partial_y^\beta\partial_x^\alpha f}\right)}{L^{q_1,q_2}}\\
        &\leq
        \norm{\mathscr{F}\left({\chEF{U}{\epsilon_1}\chEF{V}{\epsilon_2}}\right)}{L^{s_1,s_2}}
        \norm{\mathscr{F}{\left(\partial_y^\beta\partial_x^\alpha f\right)}}{L^{t_1,t_2}}.
    \end{align*}
    Note that the assumption \(\frac{1}{s_i}+\frac{1}{t_i}+\frac{1}{p_i}=2\) implies \(\frac{1}{s_i}+\frac{1}{t_i}=1+\frac{1}{q_i}\) with \(1\leq q_i,s_i,t_i\leq\infty\), which fulfills the condition of Young's convolution inequality.
    
    We can now split the two-block Fourier transform
    of the characteristic functions
    into the product of the two partial Fourier transforms
    and subsequently into a product of norms.
    With the bound \cref{eq:smooth_char_Lp_bound} on the smoothed characteristic function we get
    \begin{align*}
        \norm{\mathscr{F}\left(\chEF{U}{\epsilon_1}\chEF{V}{\epsilon_2}
        {\partial_y^\beta\partial_x^\alpha f}\right)}{L^{q_1,q_2}}
        &\leq
        \norm{\mathscr{F}_1({\chF{U}})}{L^{s_1}(\rrI)}
        \norm{\mathscr{F}_2({\chF{V}})}{L^{s_2}(\rrII)}
        \norm{\mathscr{F}{\left(\partial_y^\beta\partial_x^\alpha f\right)}}{L^{t_1,t_2}}.
    \end{align*}
    Combining this with the Hausdorff-Young inequality \cref{eq:hausdorff_young} 
    leads to an upper bound on \cref{eq:sobolev_to_sequence}
    which is independent of $\epsilon_1$ and $\epsilon_2$ and therefore holds in the limit $\epsilon_1,\epsilon_2\to 0$.
    All together, this is
    \begin{subequations}
    \begin{align*}
        \norm{f}{W_{n_1,p_1}^{n_2,p_2}(U, V)}
        &\leq C_{p_1,p_2}
        \norm{{\chF{U}}}{\FL{s_1}(\rrI)}
        \norm{{\chF{V}}}{\FL{s_2}(\rrII)}
        \sum_{\substack{|\alpha|\leq n_1\\|\beta|\leq n_2}}
        \norm{\mathscr{F}{\left(\partial_y^\beta\partial_x^\alpha f\right)}}{L^{t_1,t_2}}
        \\
        &\leq C_{p_1,p_2}
        \norm{{\chF{U}}}{\FL{s_1}(\rrI)}
        \norm{{\chF{V}}}{\FL{s_2}(\rrII)}
        \sum_{\substack{|\alpha|\leq n_1\\|\beta|\leq n_2}}
        \norm{\abs{\cdot\cdot}^{|\beta|}\abs{\cdot}^{|\alpha|} \widehat{f}(\cdo, \, \cdot\cdot)}{L^{t_1,t_2}}
        \\
        &\leq C_{n_1,n_2,p_1,p_2}
        \norm{{\chF{U}}}{\FL{s_1}(\rrI)}
        \norm{{\chF{V}}}{\FL{s_2}(\rrII)}
        \norm{\abs{\cdot\cdot}^{n_1}\abs{\cdot}^{n_2} \widehat{f}(\cdo, \, \cdot\cdot)}{L^{t_1,t_2}},
    \end{align*}
    \end{subequations}
    where $C_{n_1,n_2,p_1,p_2}$ is a non-negative constant depends only on $n_1$, $n_2$, $p_1$, and $p_2$
    which comes from the sum over all the multi-indices $\alpha$ and $\beta$.
    
    Finally,
    \begin{align*}
        \norm{f}{W_{n_1,p_1}^{n_2, p_2}(U, V)}
        &\leq C_{n_1,n_2,p_1,p_2}
        \norm{{\chF{U}}}{\FL{s_1}(\rrI)}
        \norm{{\chF{V}}}{\FL{s_2}(\rrII)}
        \norm{\abs{\cdot\cdot}^{n_1}\abs{\cdot}^{n_2} \widehat{f}(\cdo, \, \cdot\cdot)}{L^{t_1,t_2}}
        \\
        &\leq C_{n_1,n_2,p_1,p_2}
        \norm{{\chF{U}}}{\FL{s_1}(\rrI)}
        \norm{{\chF{V}}}{\FL{s_2}(\rrII)}
        \norm{\omega\widehat{f}}{L^{t_1,t_2}}
        \\
        &= C_{n_1,n_2,p_1,p_2}
        \norm{{\chF{U}}}{\FL{s_1}(\rrI)}
        \norm{{\chF{V}}}{\FL{s_2}(\rrII)}
        \norm{f}{\FL{t_1,t_2}(\omega)}.
    \end{align*}
    For polynomially controlled weight $\omega$ the Schwartz class is dense in $\FL{t_1,t_2}(\omega)$
    (see \cite[Theorem 10.1.7]{Hormander05AnalysisLinearPartial}),
    therefore, the statement holds in the limit, i.e, for any $f\in\FL{t_1,t_2}(\omega)$.
\end{proof}

The proof of \cref{lem:smoothness_lemma_high_degree}
uses a density argument with the Schwartz class;
in order to avoid this type of argument,
it is sufficient to consider the initial restriction that
$f$ belongs to
$ W_{n_1, \infty}^{n_2,\infty}
            (\omega; U, V))\cap \FL{1}(\omega; \rr{d_1},\rr{d_2}))$
            with \(n_1>d_1, n_2>d_2\) 
where $ W_{n_1, \infty}^{n_2,\infty}
            (\omega; U,  V))$
            stands for the weighted Bochner-Sobolev space
            \footnote{The weighted Bochner-Sobolev space $W_{n_1,p_1}^{n_2,p_2}(\omega,U,V)$ is the class of functions $f$ such that $\omega\partial_1^\alpha\partial_2^\beta f\in L^{p_1,p_2}(U,V)$ for all $\alpha\in\zzp{d_1}$ with $\abs{\alpha}\leq n_1$ and $\beta\in\zzp{d_2}$ with $\abs{\beta}\leq n_2$.}
or
$u \in  W_{n_1, 1}^{n_2,1}(U, V))\cap\FL{1}(\omega; \rrI,\rrII))$
            for any \(n_i\in \zzp{}\).

We immediately obtain the results
for the spectral Barron space
and Hilbert-Sobolev error
from \cite[Lemma 2]{Siegel20ApproximationRatesNeural}
as special case with a single block, $p_1=s_1=2$, and $t_1=1$.
The following corollary first shows the analogous result for two blocks,
then the result for general parameters on a single block,
and finally the existing result for a single block.
\begin{corollary}
\begin{itemize}
\item
    Under the assumptions of \cref{lem:smoothness_lemma_high_degree}
    with $p_i=s_i=2$ and $t_i=1$
    Plancherel's theorem
    results in 
    \begin{equation*}
    	\norm{f}{W_{n_1,2}^{n_2, 2}(U, V)}
            \leq 
        C_{n_1,n_2,2,2}
        \abs{U}^{1/2}
        \abs{V}^{1/2}
        \norm{f}{\FL{1}(\omega;\rrI,\rrII)},
	\end{equation*}
    where $C_{n_1,n_2,2,2}>0$ depends only on $n_1$ and $n_2$.

    \item
    For functions \(f\in  \FL{t_2}(\omega;\rr{d_1})\) with only one block of variables,
    we define \(\tilde{f}(x,y):=f(y)\),
    \(\tilde{\omega}(x,y):=\omega(y)\), 
    and \(U=(0,1)\).
    Then \(\tilde{f}\in \FL{2,t_2}(\tilde{\omega})\),
    \begin{align*}
        \norm{f}{W^{n_2,p_2}(U)}
        =\norm{\tilde{f}}{W_{0,2}^{n_2,p_2}(U,V)},
        \quad\text{and}\quad
        \norm{f}{\FL{t_2}(\omega;\rrII)}
        =c \norm{\tilde{f}}{\FL{\frac{4}{3},t_2}(\tilde{\omega};\rrI,\rrII)}
    \end{align*}
    for some $c>0$.
    The inequality then reads
    (with $s_1=t_1=\frac{4}{3}$)
    \begin{align*}
    	\norm{f}{W^{n_2,p_2}(V)}
            \leq 
        C_{n_2,p_2}
        \norm{\chF{V}}{\FL{s_2}(\rrII)}
        \norm{f}{\FL{t_2}(\omega;\rrII)},
	\end{align*}
    where $C_{n_2, p_2}>0$ depends only on $n_1$ and $p_1$.

\item
    Similar to the first case, we can specify this further to $p_2=s_2=2$, $t_2=1$, and $\omega(\xi_2)=\eabs{\xi_2}^{n_2}$,
    in which case we obtain the already known result
    \begin{align*}
    	\norm{f}{W^{n_2,2}(V)}
            \leq 
        C_{n_2,2}
        \abs{V}^{1/2}
        \norm{f}{\FL{1}(\eabs{\cdot}^{n_2};\rrI)},
	\end{align*}
    for the spectral Barron space.
    \end{itemize}
\end{corollary}

%%%%%%%%%%%%%%%%%%%%%%%%%%%%%%
%%%%%     Discussion     %%%%%
%%%%%%%%%%%%%%%%%%%%%%%%%%%%%%

\subsubsection{Structure of the Domain and Admissible Degrees}
\label{sec:structure_of_domain}
In the proof of \cref{lem:smoothness_lemma_high_degree}, we construct an upper bound on the smoothed characteristic functions.
At a first glance,
it might seem that this is an unnecessary detour;
one could simply add the assumption
that $\chF{\Omega_i}\in\FL{1}(\rr{d_i})$
for $\Omega_1=U$ and $\Omega_2=V$
and thereby trivially get $\chF{\Omega_i}\in\FL{s_i}(\rr{d_i})$
as $|\widehat{\chF{\Omega_i}}|\leq\norm{\chF{\Omega_i}}{L^1}=|\Omega_i|$
and
\begin{align*}
    \norm{\chF{\Omega_i}}{\FL{s_i}}^{s_i}
    =\norm{\widehat{\chF{\Omega_i}}^{s_i}}{L^1}
    =|\Omega_i|^{s_i}\normbig{\frac{\widehat{\chF{\Omega_i}}^{s_i}}{|\Omega_i|^{s_i}}}{L^1}
    \leq|\Omega_i|^{s_i}\normbig{\frac{\widehat{\chF{\Omega_i}}}{|\Omega_i|}}{L^1}
    =|\Omega_i|^{s_i-1}\norm{\chF{\Omega_i}}{\FL{1}}.
\end{align*}
However, doing so, would add an unnecessary strong assumption.
In order to understand this, we consider the following very simple example:
\begin{example}
    Let $\Omega=[-1/2,1/2]$.
    Then,
    \begin{align*}
        \mathscr{F}(\chF{\Omega})(\xi)
        =\frac{1}{\sqrt{2\pi}}\operatorname{sinc}(\xi/2),
    \end{align*}
    which is known to be non-integrable and therefore, $\chF{\Omega}\notin\FL{1}$.
    However, the $\operatorname{sinc}$ is integrable with respect to the $L^s$ norm for any $s>1$.
    Thus, for $s>1$, $\chF{\Omega}\in\FL{s}$.
\end{example}
This example directly extends to bounded hyperrectangles
and a finite number of disjoint unions thereof.
It shows that requiring $\chF{\Omega}\in\FL{1}$
would definitely exclude this simple case
from the use-cases of \cref{lem:smoothness_lemma_high_degree}
while it is admissible with the provided proof for any $s>1$.
However, in general there is no guarantee that $s>1$ is sufficient or whether it is even necessary for $\chF{\Omega}\in\FL{s}$.

For some general measurable and bounded domain $\Omega\subset\rr{d}$ the integrability with respect to the $L^2$ norm is trivially provided by Plancherel's theorem as
\begin{align*}
    \norm{\chF{\Omega}}{\FL{2}}=\norm{\chF{\Omega}}{L^2}=|\Omega|^{\frac{1}{2}}
\end{align*}
and for $s>2$ \cite{Ko16FourierTransformRegularity} suggests to verify the itegrability via the Hausdorff-Young inequality.
For the case $1\leq s<2$, the integrability of $\widehat{\chF{\Omega}}$ is its own field of research (see \cite{Lebedev13FourierTransformCharacteristic,Ko16FourierTransformRegularity} and references therein)
and the results of \cite{Lebedev13FourierTransformCharacteristic,Ko16FourierTransformRegularity} show that the integrability depends on the geometric properties of $\Omega$.
More specifically, they show that it depends on the geometry of the boundary of the domain.

\citeauthor{Lebedev13FourierTransformCharacteristic} mentions that for a ball $\mathcal{B}$ in $\rr{d}$ the condition 
\begin{align}
\label{eq:lebedev_condition}
    s>\frac{2d}{d+1}
\end{align}
is necessary and sufficient so that $\chF{B}\in\FL{s}$  \cite{Lebedev13FourierTransformCharacteristic}.
That is, despite having a smooth surface in
$C^\infty$,
the constraint becomes stronger for increasing $d$,
whereas the constraint for a hyperrectangle does not change, 
even though its surface is not differentiable everywhere.
In
\cite{Lebedev13FourierTransformCharacteristic}
Lebedev further elaborates on that and shows that \cref{eq:lebedev_condition} is sufficient for domains with a $C^1$ boundary and on top of that becomes necessary when the boundary is $C^2$.

Note, that the results up to now did not include any case, where $s=1$ is allowed.
This case is treated by \cite{Ko16FourierTransformRegularity} with the following proposition:

\begin{proposition}[Proposition 1.2 in \cite{Ko16FourierTransformRegularity}]
\label{prop:Fourier_Char_low_p_bound}
    Let $1 \leq s \leq 2$ and $\Omega$ be a bounded domain
    in $\rr{d}$.
    Then,
    $$
        \norm{\chF{\Omega}}{\FL{s}}=\norm{\widehat{\chF{\Omega}}}{L^s} \lesssim|\Omega|+\left(\int_0^1 \delta^{-d\left(1-\frac{s}{2}\right)}\left|(\partial \Omega)_\delta\right|^{\frac{s}{2}} \frac{d \delta}{\delta}\right)^{1 / s},
    $$
    with
    \begin{equation}\label{eq:boundaryE_dist}
        (\partial \Omega)_\delta=\{x: \operatorname{dist}(x, \partial \Omega)<\delta\}.
    \end{equation}
    Here, $\operatorname{dist}(x, \partial \Omega)$
    is defined as 
    $\inf _{\eta\in \partial \Omega} \|\eta-x\|$
    and $\|\cdot \|$ is the Euclidean norm.
\end{proposition}

Combining this theory with \cref{lem:smoothness_lemma_high_degree} yields the following proposition.
\begin{proposition}
    Under the assumptions of \cref{lem:smoothness_lemma_high_degree}
    let
    \begin{align*}
    \tau_i:=\left(\int_0^1 \delta^{-d_i\left(1-\frac{s_i}{2}\right)}\left|(\partial E_i)_\delta\right|^{\frac{s_i}{2}} \frac{d \delta}{\delta}\right)^{1 / s_i},
    \end{align*}
    with $E_1=U$ and $E_2=V$.
	Then, for any
    $f\in\FL{t_1,t_2}(\omega;\rrI,\rrII)$
    we have
	\begin{equation}\label{res:high_degree_volume_bound}
		\norm{f}{W_{n_1,p_1}^{n_2, p_2}(U, V)}
        \leq
        C_{n_1,n_2,p_1,p_2}
        \left( |U| +\tau_1\right)
        \left( |V| +\tau_2\right)
        \norm{f}{\FL{t_1,t_2}(\omega;\rrI,\rrII)},
	\end{equation}
    where $C_{n_1,n_2,p_1,p_2}>0$ depends only on $n_1, n_2$, $p_1$, and $p_2$. 
\end{proposition}

\begin{proof}
    This proposition is a direct consequence of \cref{lem:smoothness_lemma_high_degree} and \cref{prop:Fourier_Char_low_p_bound}
\end{proof}

The lower bounds on the possible choices
for the integrability exponents $s_i$
arising from the structure of the domain
translate into an upper bound on the choice of the
integrability exponents $p_i$ of the error measure.
This is due to the condition
\(s_i\), \(t_i\), and \(p_i\)
in \cref{lem:smoothness_lemma_high_degree};
the following remark sheds light on this observation.
\begin{remark}[Admissible Degrees]
    Let $s_i$ be lower bounded by $\bar{s}_i$ (i.e., $1\leq\bar{s}< s\leq 2$) and \(1\leq t_i\leq 2\leq p_i\) such that \(\frac{1}{s_i}+\frac{1}{t_i}= 2-\frac{1}{q_i}\).
    Then there is a hard upper bound on the choice of \(p_i\) given by
    \begin{align*}
        p_i
        \leq \frac{1}{2-\frac{1}{s_i}-\frac{1}{t_i}}
        <    \frac{1}{2-\frac{1}{\bar{s}_i}-\frac{1}{t_i}}
        \leq \frac{1}{1-\frac{1}{\bar{s}_i}}
        =\bar{s}_i^\prime,
    \end{align*}
    where the last inequality can be attaind with equality when choosing $t_i=1$.
\end{remark}

\subsection{Inclusion of Fourier-Lebesgue Spaces in Low-Degree Bochner-Sobolev Spaces}
\label{subsec:ConvergenceInLowOrderBochnerSobolevNorms}
In \cref{lem:smoothness_lemma_high_degree},
we covered the inclusion of anisotropic weighted Fourier-Lebesgue spaces
in Bochner-Sobolev spaces with high-degree.
This is especially interesting because of these spaces are type-2 Banach spaces.
For sake of completeness,
we also study the conjugate case of low-degree Bochner-Sobolev spaces in the following lemma:
\begin{lemma}[Inclusion in low-degree Bochner-Sobolev Space]
\label{lem:smoothness_lemma_low_degree}
    Let $1\leq p_i,t_i\leq 2$,
    $d_i, n_i \in\nn{}$ 
    for $i\in\{1,2\}$ with $t_1\geq t_2$.
    Let $\omega(x,y)$ be a polynomially moderated weight function
    defined on $\rrIxII$, elliptic with respect to
    $\eabs{x}^{n_1}\eabs{y}^{n_2}$,
    and
    $U\subset \rrI$,
    $V\subset \rrII$ are bounded domains.
    Let 
    \(f\in \FL{t_1,t_2}(\omega; \rrI, \rrII)
    \),
    then
	\begin{align}
        \label{res:low_degree_bound}
		\norm{f}{W_{n_1,p_1}^{n_2, p_2}(U, V)}
        \leq 
        C_{n_1,n_2,p_1,p_2}
        \abs{U}^{\frac{1}{p_1}+\frac{1}{t_1}-1}
        \abs{V}^{\frac{1}{p_2}+\frac{1}{t_2}-1}
        \norm{f}{\mathscr{F}L^{t_1,t_2}(\omega;\rrI,\rrII)},
	\end{align}
    where $C_{n_1,n_2,p_1,p_2}>0$ depends only on $n_1, n_2$, $p_1$ and $p_2$.
\end{lemma}

\begin{proof}
    Similar to the proof of \cref{lem:smoothness_lemma_high_degree}
    we consider a polynomially moderated weight $\omega$
    and limit to the Schwartz class
    i.e., let $f \in \mathscr{S}(\rrIxII)$.
    By \cref{eq:sobolev_to_lpq} we have
    \begin{align*}
        \norm{f}{W_{m,p_1}^{n, p_2}(U, V)}
        &\leq\sum_{|\alpha|\leq m}\sum_{|\beta|\leq n}\norm{\chF{U}\chF{V}\partial_y^\beta\partial_x^\alpha f}{L^{p_1,p_2}}.
    \end{align*}
    and use the monotonicity of \(L^p\) norms over bounded domains
    for \(p_i\leq r_i\) (\(i\in\{1,2\}\)) 
    (see \cref{prop:Lpp_inclusion}).
    That is, for \(i\in\{1,2\}\) choose \(r_i\geq p_i\)
    and set \(s_i:=\frac{r_i}{r_i-p_i}\) (\(s_i=\infty\) if \(r_i=p_i\)).
    Then, \cref{prop:Lpp_inclusion} guarantees that
    \begin{align*}
        \norm{\chF{U}\chF{V}\partial_y^\beta\partial_x^\alpha f}{L^{p_1,p_2}}
        &\leq|U|^{\frac{1}{p_1s_1}}|V|^{\frac{1}{p_2s_2}}\norm{\partial_y^\beta\partial_x^\alpha f}{L^{r_1,r_2}}.
    \end{align*}
    With $f\in\mathscr{S}$, we also have 
    $\partial_y^\beta \partial_x^\alpha f \in L^1$
    for any $\alpha\in\zzp{d_1}$ and $\beta\in\zzp{d_2}$,
    and
    \(f \in \mathscr{F}L^{1}(\omega)\).
    It follows by straightforward computations that 
    $\mathscr{F}(\partial_y^\beta \partial_x^\alpha f )\in L^1$
    for any $|\alpha|\leq m$ and $|\beta| \leq n$.
    Consequently, for any $(x,y)\in \rrIxII$, $\alpha \in \zzp{d_1}$
    and $\beta \in \zzp{d_2}$
    such that $|\alpha|\leq m$ and $|\beta| \leq n$,
    we can write 
     \begin{equation*}
        \partial_y^\beta \partial_x^\alpha f(x,y)
        \equiv 
        \mathscr{F}^{-1}\left(
                \mathscr{F}(\partial_y^\beta \partial_x^\alpha f )(\xi,\eta)
                        \right)(x,y).
     \end{equation*}
    The Hausdorff-Young inequality for mixed Lebesgue spaces \cite[Section 12]{Benedek61SpaceMixedNorm} for $2\geq t_1\geq t_2\geq 1$,
    then yields
    \begin{align*}
        \norm{\partial_y^\beta\partial_x^\alpha f}{L^{r_1,r_2}}
        &=\norm{\mathscr{F}^{-1}\left(\mathscr{F}\left(\partial_y^\beta\partial_x^\alpha f\right)\right)}{L^{r_1,r_2}}
        \leq C\norm{\mathscr{F}\left(\partial_y^\beta\partial_x^\alpha f\right)}{L^{t_1,t_2}},
    \end{align*}
    where \(C_{r_1,r_2}>0\) is a constant and \(r_i=\frac{t_i}{t_i-1}\geq 2\geq p_i\) (i.e., the above assumption on \(r_i\) is still true).
    All together, we have
    \begin{align*}
        \norm{f}{W_{n_1,p_1}^{n_2, p_2}(U, V)}^p
        &\leq C|U|^{\frac{1}{p_1s_1}}|V|^{\frac{1}{p_2s_2}}\sum_{|\alpha|\leq m}\sum_{|\beta|\leq n}\norm{\mathscr{F}\left(\partial_y^\beta\partial_x^\alpha f\right)}{L^{t_1,t_2}}.
    \end{align*}
    Expressing \(s_i\) in terms of \(t_i\) yields \(\frac{1}{s_i}=1+\frac{p_i}{t_i}-p_i\) and continuing in the same fashion as in the proof of \cref{lem:smoothness_lemma_high_degree}, this is
    \begin{align*}
        \norm{f}{W_{n_1,p_1}^{n_2, p_2}(U, V)}
        \leq C_{m,n,p}|U|^{\frac{1}{p_1}+\frac{1}{t_1}-1}|V|^{\frac{1}{p_2}+\frac{1}{t_2}-1}\norm{f}{\mathscr{F}L^{t_1,t_2}(\omega)}
    \end{align*}
    with some constant \(C_{n_1,n_2,p_1,p_2}>0\) that is only dependent on $n_1$, $n_2$, $p_1$ and $p_2$.

    We now use the same density argument as in the proof of \cref{lem:smoothness_lemma_high_degree}, to show that the result extends to all $f\in\FL{t_1,t_2}(\omega)$.
\end{proof}
An important thing to note here is that the degrees \(p\) and \(t\) can be chosen independently of one another.

\begin{corollary}
    With \(p_1=p_2=2\) and \(t_1=t_2=1\) the assumption of \cref{lem:smoothness_lemma_high_degree,lem:smoothness_lemma_low_degree} overlap and we get
    \begin{equation*}
    	\norm{f}{W_{n_1,2}^{n_2, 2}(U, V)}
            \leq 
        C_{n_1,n_2,2}
        |U|^{1/2}
        |V|^{1/2}
        \norm{f}{\mathcal{F}L^1(\omega;\rrI, \rrII)},
	\end{equation*}
    where $C_{n_1,n_2, 2}>0$ depends only on $n_1$ and $n_2$.
    This result follows directly from \cref{lem:smoothness_lemma_low_degree} and by applying Plancherel's theorem in \cref{lem:smoothness_lemma_high_degree}.
\end{corollary}

\subsection{Approximation of Fourier Lebesgue Space}
\label{sec:approximation_of_FL}
In \cref{lem:smoothness_lemma_high_degree} we showed
(for a proper choice of parameters) 
the inclusion
\begin{align*}
    \mathscr{F}L^{t_1,t_2}(\omega)
    &\subseteq
    W_{n_1,p_1}^{n_2,p_2}(U,V).
\end{align*}
These spaces inherit the Rademacher type from the underlying Lebesgue space \cite{Pelczynski86IsomorphismsAnisotropicSobolev,Hytonen17AnalysisBanachSpacesII},
which implies that $W_{n_1,p_1}^{n_2,p_2}(U,V)$ is a type-2 Banach Space for
\(2\leq p_1,p_2< \infty\) (note the strict inequality on the right side).
For more general information on the Rademacher type of Bochner spaces we refer the interested reader to \cite{Hytonen17AnalysisBanachSpacesII}.

We now use this information in order to show that any \(f\in \FL{q_1,q_2}(\omega)\) is in the variation space of the dictionary of activation functions with some certain minimal decay.
With this in mind, we can extend the proof ideas of
\cite{Barron93UniversalApproximationBounds} and \cite{Siegel20ApproximationRatesNeural}
from a single-block $L^2$- and Hilbert-Sobolev-Error, respectively, to the two-block Bochner-Sobolev-Error.

%%%%%%%%%%%%%%%%%%%%%%%%%%%%%%%%%%%%%%%%%%%%%%%%%%
%          Approximation of FL functions         %
%%%%%%%%%%%%%%%%%%%%%%%%%%%%%%%%%%%%%%%%%%%%%%%%%%
\begin{theorem}
\label{thm:approximation_sobolev_space}
    For $i\in\{1,2\}$ let $d_i,m_i,n_i \in\nn{}$,
    such that \(m_i\geq n_i\),
    and $1\leq q_i\leq 2\leq p_i< \infty$,
    with $p_1\leq p_2$.
    Let $U\subset\rrI$ and $V\subset\rrII$ be measurable and bounded domains
    with non-empty interior such that $\norm{\chF{U}}{\FL{p_1^\prime}}<\infty$
    and $\norm{\chF{V}}{\FL{p_2^\prime}}<\infty$.
    Let $\vartheta(t_1,t_2)$ be a weight over $\rr{2}$ which is
    elliptic with respect to $\eabs{t_1}^{\gamma_1}\eabs{t_2}^{\gamma_2}$ for some $\gamma_1,\gamma_2>1$,
    $\omega(\xi_1, \xi_2)$ be a weight over $\rrIxII$
    elliptic with respect to $\eabs{\xi_1}^{n_1}\eabs{\xi_2}^{n_2}$,
    and
    \begin{align*}
        \tilde{\omega}(\xi_1,\xi_2):=
            \omega(\xi_1,\xi_2)
            \eabs{\xi_1}^{(d_1+1)(1-\frac{1}{q_1})+1}
            \eabs{\xi_2}^{(d_2+1)(1-\frac{1}{q_2})+1}.
    \end{align*}

    For an activation function
    $\sigma\in W_{m_1,\infty}^{m_2,\infty}(\vartheta;\rr{2}){\setminus}\{0\}$,
    $f\in\FL{q_1,q_2}(\tilde{\omega})$,
    and sufficiently large $M>0$,
    there exits a constant $C>0$ such that
    \begin{align}
    \label{eq:approximation_bound}
        \inf_{f_{N}\in\Sigma_{N,M}(\mathbb{D}_\sigma^{d_1,d_2})}\norm{f-f_N}{W_{n_1,p_1}^{n_2,p_2}(U,V)}
        &\leq C N^{-\frac{1}{2}} \abs{U}^{1/p_1} \abs{V}^{1/p_2}
        \norm{f}{\FL{q_1,q_2}(\tilde{\omega})}
    \end{align}
    for all $N\in\nn{}$.
\end{theorem}

\begin{proof}
    We split the proof in the following 5 steps:
    \begin{enumerate}
        \item develop a representation of the phase term;
        \item represent the target function as an infinite-width shallow network;
        \item provide an upper bound on the variation norm of the target function;
        \item provide an upper bound on the supremum-norm of all functions in the dictionary;
        \item combine the upper bounds in Maurey's sampling argument.
    \end{enumerate}
    \textbf{Phase Construction:}
    For $\sigma\in W_{m_1,\infty}^{m_2,\infty}(\vartheta)$ we get $\sigma\in L^1$
    due to the ellipticity of $\vartheta$, thus,
    \begin{align*}
        \hat{\sigma}(\tau_1,\tau_2)
        &=\frac{1}{2\pi}\int_{\rr{}}
        \int_{\rr{}}\sigma(t_1,t_2)
        e^{-i(\tau_1t_1+\tau_2t_2)}d t_2 d t_1\\
        &=\frac{1}{2\pi}
        \int_{\rr{}}\int_{\rr{}}\sigma(w_1\cdot x_1+b_1,w_2\cdot x_2+b_2)
        e^{-i(\tau_1(w_1\cdot x_1+b_1)
             +\tau_2(w_2\cdot x_2+b_2))}
        d b_2 d b_1.
    \end{align*}
    By substituting the linear shift $t_i=w_i\cdot x_i + b_i$
    with some arbitrary constant $w_i,x_i\in\rr{d_i}$ ($i\in\{1,2\}$).
    With the assumptions on $\sigma$,
    there exists a tuple $(\tau_1,\tau_2)$
    such that $\tau_1,\tau_2\neq 0$ and $\hat{\sigma}(\tau_1,\tau_2)\neq0$,
    and so,
    \begin{align*}
        e^{i(\tau_1w_1\cdot x_1+\tau_2w_2\cdot x_2)}
        =\frac{1}{2\pi\hat{\sigma}(\tau_1,\tau_2)}
        \int_{\rr{}}\int_{\rr{}}\sigma(w_1\cdot x_1+b_1,w_2\cdot x_2+b_2)
        e^{-i(\tau_1b_1
             +\tau_2b_2)}
        d b_2 d b_1.
    \end{align*}

    \textbf{Shallow-Net representation:}
    We insert the representation of $e^{iw\cdot x}$ in the inverse Fourier transform of $\hat{f}$
    and consider the short notation 
    $\xi:=(\xi_1,\xi_2)$ and $b:=(b_1,b_2)$ with $\Lambda_\xi=\rrIxII$ and $\Lambda_b:=\rr{}\times\rr{}$.
    That is,
    \begin{align*}
        f(x_1,x_2)
        &=\int_{\Lambda_\xi}e^{i(\xi_1\cdot x_1+\xi_2\cdot x_2)}\hat{f}(\xi_1,\xi_2)d (\xi_1,\xi_2)\\
        &=
        \int_{\Lambda_\xi}\int_{\Lambda_b}\sigma\left(\frac{\xi_1\cdot x_1}{\tau_1}+b_1,\frac{\xi_2\cdot x_2}{\tau_2}+b_2\right)
        \frac{
            e^{-i(\tau_1b_1
                 +\tau_2b_2)}
            \hat{f}(\xi_1,\xi_2)
        }{
            2\pi\hat{\sigma}(\tau_1,\tau_2)
        }
        d (b_1,b_2) d (\xi_1,\xi_2)
    \end{align*}
    and with the parametrized affine function\footnote{Note that we consider $\tau_1$ and $\tau_2$ to be constants and, thus, do not include them in the notation of $T$.}
    \begin{align}\label{eq:affine_input}
        T\left(x_1,x_2;\xi,b\right)
        &:=\left(\frac{\xi_1\cdot x_1}{\tau_1}+b_1,\frac{\xi_2\cdot x_2}{\tau_2}+b_2\right),
    \end{align}
    and the following constant and phase-term
    \begin{align*}
        C_{\sigma}&:=\abs{2\pi\hat{\sigma}(\tau_1,\tau_2)},\qquad
        \theta(\xi,b):=
            \arg\left(
                \frac{
                    \hat{f}(\xi_1,\xi_2)
                }{
                    \hat{\sigma}(\tau_1,\tau_2)
                }
            \right)
            -\tau_1b_1-\tau_2b_2
    \end{align*}
    we have the integral representation
    \begin{align*}
        f(x_1,x_2)
        &=
        \frac{1}{C_\sigma}\int_{\Lambda_b}\int_{\Lambda_\xi}\sigma\left(T\left(x_1,x_2;\xi,b\right)\right)
        \cos\left(\theta(\xi,b)\right)
        \abs{\hat{f}(\xi)}
        d \xi d b,
    \end{align*}
    where we replaced the complex exponential by the cosine as
    we know that $f$ is real valued.
    
    However, this does not guarantee that $f$ is
    in the variation space of the dictionary $\mathbb{D}_{\sigma}^{d_1,d_2}$
    as the bound
    $\norm{f}{\mathcal{K}(\mathbb{D}_{\sigma}^{d_1,d_2})}\leq\norm{\cos(\theta(\cdot,\cdot\cdot))\hat{f}(\cdot)}{L^1(\Lambda_\xi\times \Lambda_b)}$
    on its variation norm (see \cref{prop:variation_norm})
    does not converge over $b$.
    We therefore modify the dictionary by introducing the weight
    \begin{align*}
        \tilde{\vartheta}(\xi,b)
        :=\vartheta\left(\left(\abs{b_1}-R_U\abs{\xi_1/\tau_1}\right)_+,\left(\abs{b_2}-R_V\abs{\xi_2/\tau_2}\right)_+\right),
    \end{align*}
    where
    \begin{align}
    \label{eq:domain_radius}
        R_U=\sup_{x_1\in U}\abs{x_1}
        \quad\text{and}\quad
        R_V=\sup_{x_2\in V}\abs{x_2}.
    \end{align}
    Based on this weight, we now consider the dictionary $\mathbb{D}_{\tilde{\sigma}}^{d_1,d_2}$ with
    \begin{align}
    \label{eq:scaled_activation}
        \tilde{\sigma}(x_1,x_2;\xi,b)
        :=\frac{
            \tilde{\vartheta}(\xi,b)
        }{
            \omega(\xi)
        }
        \sigma(T\left(x_1,x_2;\xi,b\right)),
    \end{align}
    where the activation function is scaled based on the parameters of the affine function.
    This leads to the representation
    \begin{align*}
        f(x_1,x_2)
        &=\frac{1}{C_\sigma}
        \int_{\Lambda_\xi}
        \int_{\Lambda_b}
        \tilde{\sigma}(x_1,x_2;\xi,b)
        \frac{\omega(\xi)}
             {\tilde{\vartheta}(\xi,b)}
        \cos\left(\theta(\xi,b)\right)\abs{\hat{f}(\xi)}
        d b d \xi
    \end{align*}
    for all $(x_1,x_2)\in \rrIxII$.
    
    \textbf{Bound for the Variation Norm:}
    The variation norm is now bounded by the $L^1$-norm
    \begin{align*}
        \norm{f}{\mathcal{K}(\mathbb{D}_{\tilde{\sigma}}^{d_1,d_2})}
        \leq\frac{1}{C_\sigma}\norm{(\omega/\tilde{\vartheta})\cos(\theta)\hat{f}}{L^1(\Lambda_\xi\times\Lambda_b)}.
    \end{align*}
    To obtain a bound,
    we first calculate the integral in direction of $b$.
    \begin{align*}
        I(\xi):=\int_{\rr{2}}\frac{1}{\tilde{\vartheta}\left(\xi,b\right)}d b
        &=
        \int_{\rr{2}}\frac{1}{\vartheta
        \left(\left(\abs{b_1}-R_U\abs{\xi_1/\tau_1}\right)_+,
        \left(\abs{b_2}-R_V\abs{\xi_2/\tau_2}\right)_+\right)}d (b_1,b_2)
        \\
        &=4\int_{\rrp{2}}\frac{1}{\vartheta\left(\left(b_1-R_U\abs{\xi_1/\tau_1}\right)_+,\left(b_2-R_V\abs{\xi_2/\tau_2}\right)_+\right)}d (b_1,b_2)
    \end{align*}
    Note that we integrate over a function that is
    constant along $b_1$ for $b_1<\abs{\frac{R_U\xi_1}{\tau_1}}$
    and constant along $b_2$ for $b_2<\abs{\frac{R_V\xi_2}{\tau_2}}$.
    In all other cases, 
    it is a shifted version of $\vartheta$.
    Reverting the shift leads to

    \begin{align*}
        I(\xi)
        &=
        4\absbig{\frac{R_U\xi_1}{\tau_1}}
        \absbig{\frac{R_V\xi_2}{\tau_2}}\frac{1}{\vartheta
        \left(0,0\right)}
        +4\absbig{\frac{R_V\xi_2}{\tau_2}}
        \int_{\rrp{}}\frac{1}{\vartheta\left(b_1,0\right)}db_1
        \\
        &\qquad+4\absbig{\frac{R_U\xi_1}{\tau_1}}
        \int_{\rrp{}}\frac{1}{\vartheta\left(0,b_2\right)}db_2
        +4\int_{\rrp{}\times\rrp{}}\frac{1}{\vartheta
        \left(b_1,b_2\right)}d(b_1,b_2).
    \end{align*}
    Furthermore, the ellipticity of $\vartheta$
    (i.e., $\vartheta(b_1,b_2)\geq c\eabs{b_1}^{\gamma_1}\eabs{b_2}^{\gamma_2}\geq \frac{c}{2}(1+\abs{b_1})^{\gamma_1}(1+\abs{b_2})^{\gamma_2}$ for some $c>0$)
    finally results in the upper bound on the integrals
    \begin{align*}
        I(\xi)
        &\leq
        8c\left(\absbig{\frac{R_U\xi_1}{\tau_1}}+\frac{1}{\gamma_1-1}\right)
        \left(\absbig{\frac{R_V\xi_2}{\tau_2}}+\frac{1}{\gamma_2-1}\right)
        \\
        &\leq C_{U,V,m_1,m_2}\eabs{\xi_1}\eabs{\xi_2},
    \end{align*}
    where the constant in the last inequality is given by
    \begin{align*}
        C_{U,V,m_1,m_2}
        =
        8c\cdot\max\left\{\absbig{\frac{R_U}{\tau_1}},\frac{1}{\gamma_1-1}\right\}
        \max\left\{\absbig{\frac{R_V}{\tau_2}},\frac{1}{\gamma_2-1}\right\}.
    \end{align*}
    Using this in the variation norm leads to
    \begin{align*}
        \norm{f}{\mathcal{K}(\mathbb{D}_{\tilde{\sigma}})}
            &\leq\frac{1}{C_\sigma}\norm{\omega\hat{f}/\tilde{\vartheta}}{L^1(\Lambda_\xi\times\Lambda_b)}
            =\frac{1}{C_\sigma}\int_{\Lambda_\xi}\int_{\Lambda_b}\frac{\omega(\xi)\abs{\hat{f}(\xi)}}
                    {\tilde{\vartheta}(\xi,b)} db d\xi
        \\
        &=\frac{1}{C_{\sigma}}
            \int_{\Lambda_\xi}\omega(\xi)\abs{\hat{f}(\xi)}
            \int_{\Lambda_b}\frac{1}{\tilde{\vartheta}\left(\xi,b\right)}d b d \xi
        =\frac{1}{C_{\sigma}}
            \int_{\Lambda_\xi}\omega(\xi)\abs{\hat{f}(\xi)}I(\xi)d(\xi)\\
        &\leq C_{U,V,\sigma,m_1,m_2}
            \int_{\Lambda_\xi}\omega(\xi_1,\xi_2)\eabs{\xi_1}\eabs{\xi_2}\abs{\hat{f}(\xi_1,\xi_2)}d(\xi_1,\xi_2).
    \end{align*}
    We now define the probability measure
    \begin{align*}
        \nu(\xi_1,\xi_2)&:=\frac{c_1c_2}{\eabs{\xi_1}^{d_1+1}\eabs{\xi_2}^{d_2+1}}
    \qquad\text{with}\qquad
        c_i^{-1}:=\normbig{\eabs{\cdot}^{-(d_i+1)}}{L^1(\rr{d_i})}
    \end{align*}
    and the marginal probability measures
    \begin{align*}
        \nu_1(\xi_1)=\frac{c_1}{\eabs{\xi_1}^{d_1+1}}
        \qquad\text{and}\qquad
        \nu_2(\xi_2)=\frac{c_2}{\eabs{\xi_2}^{d_2+1}}.
    \end{align*}
    
    Consequently, we can continue the bound on the variation norm with
    \begin{align*}
        \psi(\xi_1, \xi_2)
        :=
        \omega(\xi_1,\xi_2)\eabs{\xi_1}^{d_1+2}\eabs{\xi_2}^{d_2+2}
    \end{align*}
    as
    \begin{align}
    \label{eq:FL_q_tilde_inclusion_I}
    \begin{split}
        \norm{f}{\mathcal{K}(\mathbb{D}_{\tilde{\sigma}})}
        &\leq C_{U,V,\sigma,m_1,m_2}
            \int_{\Lambda_\xi}\absbig{
                \omega(\xi_1,\xi_2)
                \eabs{\xi_1}\eabs{\xi_2}
                \hat{f}(\xi_1,\xi_2)
            }
            d(\xi_1,\xi_2)\\
        &= \frac{C_{U,V,\sigma,m_1,m_2}}{c_1c_2}
            \int_{\Lambda_\xi}\absbig{
                \omega(\xi_1,\xi_2)
                \eabs{\xi_1}^{d_1+2}\eabs{\xi_2}^{d_2+2}
                \hat{f}(\xi_1,\xi_2)
            }
            d\nu(\xi_1,\xi_2)\\
        &= \frac{C_{U,V,\sigma,m_1,m_2}}{c_1c_2}
            \left(\int_{\rrI}
                \left(\int_{\rrII}\absbig{
                    \psi(\xi_1,\xi_2)
                    \hat{f}(\xi_1,\xi_2)
                }
                d\nu_2(\xi_2)\right)^{\frac{q_2}{q_2}}
            d\nu_1(\xi_1)\right)^{\frac{q_1}{q_1}}\\
        &\leq \frac{C_{U,V,\sigma,m_1,m_2}}{c_1c_2}
            \left(\int_{\rrI}
                \left(\int_{\rrII}\absbig{
                    \psi(\xi_1,\xi_2)
                    \hat{f}(\xi_1,\xi_2)
                }^{q_2}
                d\nu_2(\xi_2)\right)^{\frac{q_1}{q_2}}
            d\nu_1(\xi_1)\right)^{\frac{1}{q_1}}
    \end{split}
\end{align}
    by using Jensen's inequality (see \cite{Rudin13RealComplexAnalysis}) on both integrals.
    We now proceed to express the right side in terms of the $\FL{}$-norm.
\begin{align}
    \label{eq:FL_q_tilde_inclusion_II}
    \begin{split}
        \norm{f}{\mathcal{K}(\mathbb{D}_{\tilde{\sigma}})}
        \kern-2em&\kern2em\leq \frac{C_{U,V,\sigma,m_1,m_2}}{c_1c_2}
            \left(\int_{\rrI}
                \left(\int_{\rrII}
                    \absbig{
                        \psi(\xi_1,\xi_2)
                        \hat{f}(\xi_1,\xi_2)
                    }^{q_2}
                d\nu_2(\xi_2)\right)^{\frac{q_1}{q_2}}
            d\nu_1(\xi_1)\right)^{\frac{1}{q_1}}
        \\
        &= \frac{C_{U,V,\sigma,m_1,m_2}}{c_1c_2}
            \left(\int_{\rrI}
                \left(\int_{\rrII}
                    \absbig{
                        \psi(\xi_1,\xi_2)
                        \hat{f}(\xi_1,\xi_2)
                    }^{q_2}
                \frac{c_2}{\eabs{\xi_2}^{d_2+1}}d\xi_2\right)^{\frac{q_1}{q_2}}
            \frac{c_1}{\eabs{\xi_1}^{d_1+1}}d\xi_1\right)^{\frac{1}{q_1}}
        \\
        &= \frac{C_{U,V,\sigma,m_1,m_2}}
                {c_1^{1-\frac{1}{q_1}}
                 c_2^{1-\frac{1}{q_2}}}
            \left(\int_{\rrI}
                \left(\int_{\rrII}
                    \absbig{
                        \tilde{\omega}(\xi_1,\xi_2)
                        \hat{f}(\xi_1,\xi_2)
                    }^{q_2}
                d\xi_2\right)^{\frac{q_1}{q_2}}
            d\xi_1\right)^{\frac{1}{q_1}}
        \\
        &=
        \frac{C_{U,V,\sigma,m_1,m_2}}
                {c_1^{1-\frac{1}{q_1}}
                 c_2^{1-\frac{1}{q_2}}}
        \norm{f}{\FL{q_1,q_2}(\tilde{\omega})}
    \end{split}
\end{align}
    with
    \begin{align*}
        \tilde{\omega}(\xi_1,\xi_2)
            &:=
            \psi(\xi_1,\xi_2)
            \eabs{\xi_1}^{-(d_1+1)\frac{1}{q_1}}
            \eabs{\xi_2}^{-(d_2+1)\frac{1}{q_2}}
            \\
            &\phantom{:}=
            \omega(\xi_1,\xi_2)
            \eabs{\xi_1}^{(d_1+1)(1-\frac{1}{q_1})+1}
            \eabs{\xi_2}^{(d_2+1)(1-\frac{1}{q_2})+1}.
    \end{align*}   
    
    \textbf{Bound for the Dictionary:} Let $\tilde{\sigma}$
    from \eqref{eq:scaled_activation} and $(\xi,b)\in\Lambda_\xi\times\Lambda_b$;
    to see that $\norm{\tilde{\sigma}(\cdot,\cdot\cdot;\xi,b)}{W_{n_1,p_1}^{n_2,p_2}}$
    exists and to obtain a bound
    we first write $T\left(x_1,x_2;\xi,b\right)=(T_1\left(x_1,x_2;\xi,b\right),T_2\left(x_1,x_2;\xi,b\right))\in\rr{2}$ for $T$ from \cref{eq:affine_input}
    and verify for $i \in \{1,2\}$ that
    \begin{align}
    \label{eq:activation_arg_bound}
        T_i\left(x_1,x_2;\xi,b\right)
        =\abs{x_i\cdot\xi_i/\tau_i+b_i}
        &\geq \left(\abs{b_i}-\abs{x_i\cdot\xi_i/\tau_i}\right)_+
        \geq \left(\abs{b_i}-R_{\Omega_i}\abs{\xi_i/\tau_i}\right)_+,
    \end{align}
    where $\Omega_1=U$, $\Omega_2=V$,
    and $R_{\Omega_i}$ as in \cref{eq:domain_radius}.
    Together with the assumption $\sigma\in W_{m_1,\infty}^{m_2,\infty}(\vartheta)$
    we get
    \begin{align*}
        \left(\partial_1^{i_1}\partial_2^{i_2}\sigma\right)(T\left(x_1,x_2;\xi,b\right))
        \leq\frac{C_{\sigma,\vartheta}}{\vartheta(T\left(x_1,x_2;\xi,b\right))}
        \leq\frac{C_{\sigma,\vartheta}}{\tilde{\vartheta}(\xi,b)}
    \end{align*}
    for all $(x_1,x_2)\in U\times V$ and
    $i_1\in\zzp{}, i_2\in\zzp{}$ such that $i_1\leq m_1, i_2\leq m_2$
    where $\partial_i$ refers to the derivative with respect to the $i$-{th}
    variable.
    The constant $C_{\sigma,\vartheta}$ in this inequality is given by
    \begin{align*}
        C_{\sigma,\vartheta}
            =   \max_{\substack{0\leq i_1\leq m_1\\0\leq i_2\leq m_2}}
                \sup_{t\in\rr{2}}
                \vartheta(t)
                \partial_{t_1}^{i_1}\partial_{t_2}^{i_2}\sigma(t).
    \end{align*}
    
    This leads to
    \begin{align*}
        \norm
        {\partial_{x_1}^\alpha\partial_{x_2}^\beta\tilde{\sigma}(T\left(\cdot,\cdot\cdot;\xi,b\right))}{L^{p_1,p_2}(U, V)}
        \kern-10em&\kern10em=
        \frac{\tilde{\vartheta}(\xi,b)}{\omega(\xi)}
        \norm{\partial_{x_1}^\alpha\partial_{x_2}^\beta\sigma(T\left(\cdot,\cdot\cdot;\xi,b\right))}
        {L^{p_1,p_2}(U, V)}
        \\
        &=\frac{\tilde{\vartheta}(\xi,b)}{\omega(\xi)}
            \absbig{\frac{\xi_1^\alpha}{\tau_1^{\abs{\alpha}}}
                    \frac{\xi_2^\beta}{\tau_2^{\abs{\beta}}}}
            \normbig{\left(\partial_1^{\abs{\alpha}}\partial_2^{\abs{\beta}}\sigma\right)
                     (T\left(\cdot,\cdot\cdot;\xi,b\right))
                }{L^{p_1,p_2}(U, V)}\\
        &\leq\tilde{\vartheta}(\xi,b)
            \abs{\tau_1}^{-\abs{\alpha}}
                    \abs{\tau_2}^{-\abs{\beta}}
            \normbig{\left(\partial_1^{\abs{\alpha}}\partial_2^{\abs{\beta}}\sigma\right)
                     (T\left(\cdot,\cdot\cdot;\xi,b\right))
                }{L^{p_1,p_2}(U, V)}\\
        &\leq C_{\sigma,\vartheta}\tilde{\vartheta}(\xi,b)
            \abs{\tau_1}^{-\abs{\alpha}}
                    \abs{\tau_2}^{-\abs{\beta}}
            \normbig{\frac{1}{\tilde{\vartheta}(\xi,b)}
                }{L^{p_1,p_2}(U, V)}\\
        &\leq C_{\sigma,\vartheta}\abs{U}^{1/p_1}\abs{V}^{1/p_2}
            \abs{\tau_1}^{-\abs{\alpha}}
                    \abs{\tau_2}^{-\abs{\beta}}
    \end{align*}
    and finally to the bound on the Sobolev norm
    \begin{align*}
        \norm{\tilde{\sigma}(\cdot,\cdot\cdot;\xi,b)}{W_{n_1,p_1}^{n_2,p_2}(U, V)}
        \kern-5em&\kern5em=\Bigg(\sum_{\abs{\alpha}\leq n_1}
        \Bigg(\sum_{\abs{\beta}\leq n_2}
            \norm{\partial_{x_1}^\alpha
                \partial_{x_2}^\beta
                \tilde{\sigma}(\cdot,\cdot\cdot;\xi,b)
                }{L^{p_1,p_2}(U, V)}^{p_2}
        \Bigg)^{\frac{p_1}{p_2}}
        \Bigg)^{\frac{1}{p_1}}\\
        &\leq C_{\sigma,\vartheta}\Bigg(\sum_{\abs{\alpha}\leq n_1}
        \Bigg(\sum_{\abs{\beta}\leq n_2}
            \abs{U}^{p_2/p_1}\abs{V}^{p_2/p_2}
            \abs{\tau_1}^{-\abs{\alpha}p_2}
                    \abs{\tau_2}^{-\abs{\beta}p_2}
        \Bigg)^{\frac{p_1}{p_2}}
        \Bigg)^{\frac{1}{p_1}}\\
        &=C_{\sigma,\vartheta}\abs{U}^{1/p_1}\abs{V}^{1/p_2}
        \Bigg(\sum_{\abs{\alpha}\leq n_1}
            \abs{\tau_1}^{-p_1\abs{\alpha}}
        \Bigg)^{\frac{1}{p_1}}
        \Bigg(\sum_{\abs{\beta}\leq n_2}
            \abs{\tau_2}^{-p_2\abs{\beta}}
        \Bigg)^{\frac{1}{p_2}}.
    \end{align*}

    \textbf{The final bound:}
    With the assumption $f\in\FL{q_1,q_2}(\tilde{\omega})$ and
    \cref{eq:FL_q_tilde_inclusion_I} and \cref{eq:FL_q_tilde_inclusion_II},
    we get that
    $f\in\FL{1}(\omega(\cdot,\cdot\cdot)\eabs{\cdot}\eabs{\cdot\cdot})\subseteq\FL{1}(\omega)$.
    With the additional assumption that $p_1\leq p_2$, we can therefore apply \cref{lem:smoothness_lemma_high_degree} with $t_i=1$ and $s_i=p_i^\prime$
    to get $f\in W_{n_1,p_1}^{n_2,p_2}(U,V)$.
    This is a type-2 Banach space
    for $2\leq p_1,p_2<\infty$
    and thus, with Maurey's approximation bound (see \cref{prop:approximation_type2})
    and $M:=\norm{f}{\mathcal{K}(\mathbb{D}_{\tilde{\sigma}}^{d_1,d_2})}$,
    the final bound is
    \begin{align*}
        &\kern-2em\inf_{f_N\in\Sigma_{N,M}(\mathbb{D}_{\tilde{\sigma}}^{d_1,d_2})}
        \norm{f-f_N}{W_{n_1,p_1}^{n_2,p_2}(U, V)}
        \\
        &\lesssim N^{-1/2}
            \sup_{\xi,b}
                \norm{\tilde{\sigma}(\cdot,\cdot\cdot;\xi,b)}{W_{n_1,p_1}^{n_2,p_2}(U, V)}
            \norm{f}{\mathcal{K}(\mathbb{D}_{\tilde{\sigma}}^{d_1,d_2})}\\
            &=N^{-1/2}
                C
                \abs{U}^{1/p_1}\abs{V}^{1/p_2}
                \norm{f}{\FL{q_1,q_2}(\tilde{\omega})}.
    \end{align*}
\end{proof}

The constant $C$ in \cref{thm:approximation_sobolev_space} can be written as (note that $\frac{1}{q_i^\prime}=1-\frac{1}{q_i}$)
\begin{align*}
\label{eq:constant_splitting}
    C
    &=C_{U,V,\sigma,m_1,m_2}
        c_1^{-\frac{1}{q_1^\prime}}
        c_2^{-\frac{1}{q_2^\prime}}
        C_{\sigma,\vartheta}
        \kappa_1^{\frac{1}{p_1}}
        \kappa_2^{\frac{1}{p_2}},
\end{align*}
where 
\begin{align*}
    C_{U,V,\sigma,m_1,m_2}
    &=
    \frac{8}{\abs{2\pi\hat{\sigma}(\tau_1,\tau_2)}}\max\left\{\absbig{\frac{R_U}{\tau_1}},\frac{1}{\gamma_1-1}\right\}
    \max\left\{\absbig{\frac{R_V}{\tau_2}},\frac{1}{\gamma_2-1}\right\},\\
    C_{\sigma,\vartheta}
        &=  \max_{\substack{0\leq i_1\leq m_1\\0\leq i_2\leq m_2}}
            \sup_{t\in\rr{2}}
            \vartheta(t)
            \partial_{t_1}^{i_1}\partial_{t_2}^{i_2}\sigma(t),
\end{align*}
\begin{align*}
    c_i&=\normbig{\eabs{\cdot}^{-(d_i+1)}}{L^1(\rr{d_i})}^{-1},
    \qquad\text{and}\qquad
    \kappa_i=
    \sum_{\substack{\abs{\alpha}\leq n_i
    \\
    \alpha\in \zzp{d_i}}}
            \abs{\tau_i}^{-p_i\abs{\alpha}},\qquad i\in \{1, 2\}.
\end{align*}

We will now show that $c_i^{-\frac{1}{q_i^\prime}}\kappa_i^{\frac{1}{p_1}}$
can be bounded independent of the dimension,
which leaves only the dependence on $R_U$ and $R_V$.

\begin{proposition}
\label{prop:breaking_curse_of_dim}
    If $n_i$ grows at most linearly with $d_i$,
    then the constant $C$ in \cref{thm:approximation_sobolev_space}
    is dimension-independent apart from the structure of the domains $U$ and $V$.
\end{proposition}
\begin{proof}
To show this independence,
we now consider the individual components of the constant as listed in \cref{eq:constant_splitting}.
Note that $C_{U,V,\sigma,m_1,m_2}$ 
is independent of the dimensions except for the
implicit dependence on the supremum $R_U$ and $R_V$ of $U$ and $V$.
For $\kappa_i$ we develop an upper bound by counting the
possible combinations for $\abs{\alpha}=k$ for $k\in\{0,\ldots,n_i\}$
and $i\in \{1, 2\}$
\begin{align*}
    \kappa_i&=
    \sum_{\substack{\abs{\alpha}\leq n_i
    \\
    \alpha\in \zzp{d_i}}}
            \abs{\tau_i}^{-p_i\abs{\alpha}}
    =
    \sum_{k=0}^{n_i}
        \abs{\tau_i}^{-p_ik}
    \sum_{\substack{\abs{\alpha}=k
    \\
    \alpha\in \zzp{d_i}}}
        1
    =
    \sum_{k=0}^{n_i}
        \binom{k+d_i-1}{k}
        \abs{\tau_i}^{-p_ik}.
\end{align*}
By taking the upper bound $\frac{(k+d_i-1))!}{(d_i-1)!}\leq (k+d_i-1)^k$
on the binomial coefficient, bounding $k\leq n_i$,
and taking the limit of the summation we get
\begin{align*}
    \kappa_i
    &\leq
    \sum_{k=0}^{n_i}
        \frac{(k+d_i-1)^k}{k!}
        \abs{\tau_i}^{-p_ik}
    \leq
    \sum_{k=0}^{n_i}
        \frac{((n_i+d_i-1)\abs{\tau_i}^{-p_i})^k}{k!}
        \\
    &\leq
    \sum_{k=0}^{\infty}
        \frac{((n_i+d_i-1)\abs{\tau_i}^{-p_i})^k}{k!}
    =e^{(n_i+d_i-1)\abs{\tau_i}^{-p_i}}.
\end{align*}
Further, via a transformation to spherical coordinates
(see \cite[(6.131) Examples (c)]{Stromberg15IntroductionClassicalReal}, \cite[Section 3]{Smith89HowSmallUnit})
we can calculate
\begin{align*}
    c_i^{-1}&=\normbig{\eabs{\cdot}^{-(d_i+1)}}{L^1(\rr{d_i})}
    =\int_{\rr{d_i}}(1+\abs{x}^2)^{-\frac{d_i+1}{2}}dx\\
    &=\frac{2\pi^\frac{d_i}{2}}{\Gamma(\frac{d_i}{2})}
        \int_{\rr{}_+}\frac{r^{d_1-1}}{(1+r^2)^{\frac{d_i+1}{2}}}dr
    =\frac{\pi^\frac{d_i}{2}}{\Gamma(\frac{d_i}{2})}
        \int_{\rr{}_+}\frac{t^{\frac{d_1}{2}-1}}{(1+t)^{\frac{d_i+1}{2}}}dt
    \\
    &=\frac{\pi^\frac{d_i}{2}}{\Gamma(\frac{d_i}{2})}B\left(\frac{d_i}{2},\frac{1}{2}\right)
    =\frac{\pi^\frac{d_i+1}{2}}{\Gamma(\frac{d_i+1}{2})}
\end{align*}
where $B(\cdot,\cdot\cdot)$ is the Euler beta-function.
Combining $c_i$ and $\kappa_i$ leads to
\begin{align*}
    &\kern-1em\left(c_i^{-\frac{1}{q_i^\prime}}\kappa_i^{\frac{1}{p_1}}\right)^{q_i^\prime}
    \leq
    \frac{2\pi^\frac{d_i+1}{2}}{\Gamma(\frac{d_i+1}{2})}
    \exp\left({\frac{q_i^\prime}{p_i}(n_i+d_i-1)\abs{\tau_i}^{-p_i}}\right)
    ,
\end{align*}
for which we can find a uniform bound, independend of $d_i$.
To do so, we first assume that 
the order of differentiability
grows at most linearly with the number of dimensions
(i.e., $n_i\leq\delta_i d_i$ for some $\delta_i>0$), thereby
\begin{align}
    \frac{q_i^\prime}{p_i}(n_i+d_i-1)\abs{\tau_i}^{-p_i}
    &\leq \frac{q_i^\prime}{p_i}(\delta_id_i+d_i-1)\abs{\tau_i}^{-p_i}
    \leq \frac{q_i^\prime}{p_i}(\delta_i+1)d_i\abs{\tau_i}^{-p_i}\\
    &\leq 2\frac{q_i^\prime}{p_i}(\delta_i+1)\frac{d_i+1}{2}\abs{\tau_i}^{-p_i}
    =:\gamma_i\frac{d_i+1}{2}.
\end{align}
Second, we observe from
the lower bound in the Stirling formula
that $x\Gamma(x)\geq \left(\frac{x}{e}\right)^{x}$.
This means that
\begin{align*}
    \frac{2\pi^\frac{d_i+1}{2}}{\Gamma(\frac{d_i+1}{2})}
    &\exp\left({\frac{q_i^\prime}{p_i}(n_i+d_i-1)\abs{\tau_i}^{-p_i}}\right)
    \leq
    \frac{2\pi^\frac{d_i+1}{2}}
         {\Gamma(\frac{d_i+1}{2})}
         \exp\left(\gamma_i\frac{d_i+1}{2}\right)
    \\
    &\leq
    \frac{(d_i+1)\pi^\frac{d_i+1}{2}}
         {\left(\frac{d_i+1}{2}\right)^{\frac{d_i+1}{2}}}
        \exp\left(\gamma_i\frac{d_i+1}{2}+\frac{d_i+1}{2}\right)
    \\
    &=
    \left(\frac{2\pi
          \exp\left(\gamma_i+1\right)}{d_i+1}\right)^{\frac{d_i+1}{2}}(d_i+1).
\end{align*}
Note that the numerator in the fraction is constant with respect to $d_i$ for any choice of parameters.
Therefore, for sufficiently large number $d_i$,
it can be upper bounded by a geometric decay,
which decays faster than the remaining linear term.
Thus, the constant $C$ can be uniformly bounded
for all $d_i\in\nn{}$ by taking the maximum over $d_i$.
\end{proof}

A very interesting observation for the
the spectral Barron space is that it has
a similar structure to the (fractional) Hilbert-Sobolev space.
To see this, consider the domain $\Omega=\rr{d}$,
then the infimum is over a set with one element, i.e.,
\begin{align*}
    \norm{f}{\mathscr{B}_{\omega}(\rr{d})}
    =
    \int_{\rr{d}}\omega(\xi)\abs{\hat{f}(\xi)}d\xi.
\end{align*}
For the special case $\omega(\xi)=(1+\abs{\xi}^{2})^{\frac{s}{2}}$
this is then the $L^1$-equivalent of the Sobolev space $H^s$ \cite[Definition 7.9.1]{Hormander98AnalysisLinearPartial}.
With this in mind,
we can view the weighted Fourier-Lebesgue spaces
as an interpolatn between the spectral Barron space
and the fractional Sobolev space.
For a bounded domain $\Omega\subseteq \rr{d}$, \citeauthor{Barron93UniversalApproximationBounds} showed that
$H^{\lfloor\frac{d}{2}\rfloor+2}(\rr{d})
\subseteq
\mathscr{B}_{\abs{\cdot}}(\rr{d})$
cf., \cite[Property 15]{Barron93UniversalApproximationBounds},
which allows to measure the approximation error in $L^2$.
Based on \cref{thm:approximation_sobolev_space}
we can extend this result
such that we allow the error measure
to an arbitrary Hilbert-Sobolev norm,
while simply asking for a linear increase in the regularity of the target function.

\begin{corollary}[Approximation of Hilbert-Sobolev Spaces]
    Let
    $f\in 
    W_{m_1,2}^{m_2,2}(\rrI, \rrII)$
    with $m_i=n_i+\lfloor\frac{d_i}{2}\rfloor+2$
    $U\subset\rrI$, $V\subset\rrII$ be bounded domains,
    and $M>0$ be sufficiently large. Let $\sigma$ be the activation function
    from  \cref{thm:approximation_sobolev_space}.
    Then
    there exits a constant $C>0$ such that
    \begin{align}
    \label{eq:cor_approximation_bound}
        \inf_{f_{N}\in \Sigma_{N,M}(\mathbb{D}_\sigma^{d_1,d_2})}\norm{f-f_N}{W_{n_1,2}^{n_2,2}(U,V)}
        &\leq C N^{-\frac{1}{2}} \abs{U}^{1/2} \abs{V}^{1/2}
        \norm{f}{W_{m_1,2}^{m_2,2}(\rrI, \rrII)}
    \end{align}
    for all $N\in\nn{}$.
\end{corollary}

\begin{proof}
    With $\omega(\xi)=\eabs{\xi_1}^{n_1}\eabs{\xi_2}^{n_2}$ and $\tilde{\omega}$ defined as in \cref{thm:approximation_sobolev_space} we observe
    \begin{align*}
        \tilde{\omega}(\xi_1,\xi_1)
        &=\eabs{\xi_1}^{n_1+d_1/2+3/2}\eabs{\xi_2}^{n_2+d_2/2+3/2}\\
        &\leq\eabs{\xi_1}^{n_1+\lfloor d_1/2\rfloor+2}\eabs{\xi_2}^{n_2+\lfloor d_2/2\rfloor+2}\\
        &=\eabs{\xi_1}^{m_1}\eabs{\xi_2}^{m_2}
    \end{align*}
    and therefore
    \begin{align*}
        \norm{f}{\FL{2}(\tilde{\omega})}
        \leq \norm{f}{\FL{2}(\eabs{\cdot}^{m_1}\eabs{\cdot\cdot}^{m_2})}
        =\norm{f}{W_{m_1,2}^{m_2,2}(\rrI, \rrII)}.
    \end{align*}
    For
    $f\in
    W_{m_1,2}^{m_2,2}(\rrI, \rrII)
    $
    we thus get
    $f\in
    \FL{2}(\tilde{\omega})
    $
    and with \cref{thm:approximation_sobolev_space} there is a $C>0$ such that
    \begin{align*}
        \inf_{f_{N}\in\Sigma_{N,M}(\mathbb{D}_\sigma^{d_1,d_2})}\norm{f-f_N}{W_{n_1,2}^{n_2,2}(U,V)}
        &\leq C N^{-\frac{1}{2}} \abs{U}^{\frac{1}{2}} \abs{V}^{\frac{1}{2}}
        \norm{f}{\FL{2}(\tilde{\omega})}\\
        &\leq C N^{-\frac{1}{2}} \abs{U}^{\frac{1}{2}} \abs{V}^{\frac{1}{2}}
        \norm{f}{W_{m_1,2}^{m_2,2}(\rrI, \rrII)}.
    \end{align*}
\end{proof}

\section{Experiments: Approximating Functions with Anisotropic Differentiability}
\label{sec:experiments}

In our initial assumption on the activation function (see \cref{sec:prelim:VariationSpace}),
we assumed that it is a function of two variables.
This assumption deviates from the typical convention
in the machine learning community,
where the activation function acts on a one-dimensional input space.
To illustrate the gain in approximation accuracy
that is obtained by our two-block structure,
we consider the following simple case:
Let $f$ be a function with two input variables
and assume that for the training we use
the Hilbert-Bochner-Sobolev norm with $n_1<n_2$ as loss function.

The example that we consider will be the function
\begin{align}
    f(t,x)=e^{-\abs{t}-\abs{x}^3}
\end{align}
which is a solution of the two-phase heat equation
\begin{align}
    (\partial_t-\partial_x^2+\operatorname{sign}(t)+9\abs{x}^4-6\abs{x})f(t,x)=0
    \quad\text{with}\quad
    f(0,x)=e^{-|x|^3}
\end{align}
over the domain $U=V=[-1,1]$.
In order to approximate this function, we consider two models for the activation function,
namely, a conventional single-block model and a novel two-block model.
More precisely, for width $N\in\nn{}$, the single-block model is given by the ReLU$^{m}$-network
\begin{align}
    \Phi_1(t,x)=w\cdot (w_t t + w_x x + b)_+^m + c,
\end{align}
with the trainable parameters $w,w_t,w_x,b\in\rr{N}$ and $c\in\rr{}$
and the two-block model is given by the mixed network 
\begin{align}
    \Phi_2(t,x)=w\cdot \left[(w_t t + b_t)_+^{m_1}(w_x x + b_x)_+^{m_2}\right] + c,
\end{align}
with the trainable parameters $w,w_t,w_x,b_t,b_x\in\rr{N}$ and $c\in\rr{}$.
We see that $\Phi_1$ and $\Phi_2$ have $4N+1$ and $5N+1$ trainable parameters, respectively.
For a fair comparison between the two models, we thus fix the number of trainable parameters instead of the width of the network.

For the loss function we consider $n_1=0$ derivatives for $t$ and $n_2=1$ derivatives for $x$.
In order to allow optimization via gradient decent while having the first derivative inside the objective, we fix $m=m_2=2$ and $m_1=1$.

With this setting,
we train $10$ random initializations
of both models for 201 and 401 parameters, respectively.
The resulting average temporal evolution
of the of the logarithmic loss function
is displayed in \cref{fig:loss}.
For each set of hyperparameters,
we plot the mean value along with an error band of one standard deviation in both directions.
The temporal evolution of the loss was smoothed (i.e., remove ripples which are due to the optimization algorithms)
by applying a cumulative minimum before calculating the average and the standard deviation.
From this evolution it is clear,
that the two-block model outperforms the single block model
for each of the two settings for the number of parameters.
\begin{figure}[H]
    \centering
    \includegraphics[page=1]{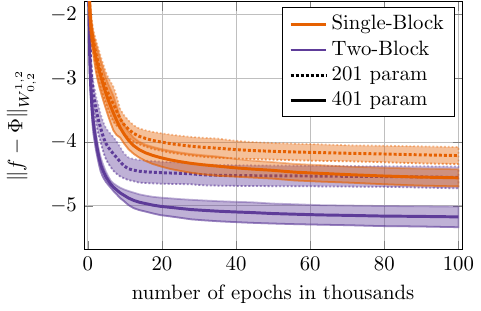}
    \caption{Temporal evolution of the logarithmic loss for the single-block model and the two-block model with 201 and 401 parameters.}
    \label{fig:loss}
\end{figure}

To further illustrate this difference,
we present a contour plot of the resulting approximation
with 401 parameters in \cref{fig:prediction}.
In this figure we see that the contour lines of the target function show sharp edges which can be approximated very well by the ReLU-part of the two-block activation function.
Contrary to that,
the ReLU$^2$-structure of the single block network
is not capable of fitting
to this structure of the target function.
\begin{figure}[H]
    \centering
    \includegraphics[page=2]{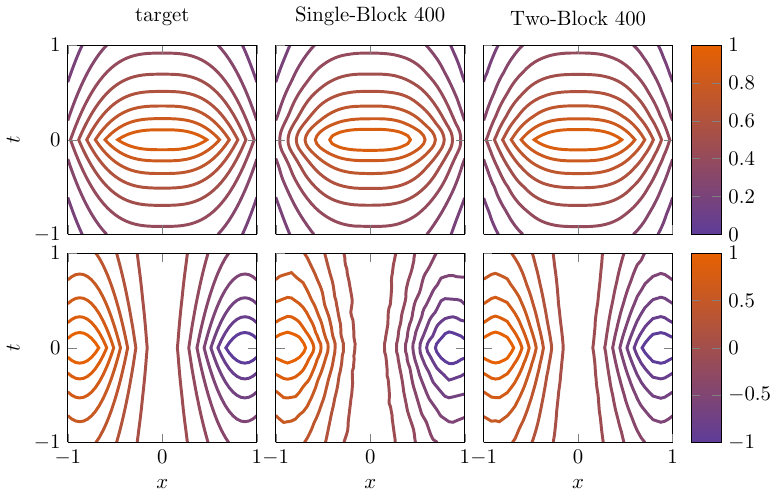}
    \caption{Countour plot of the target function $f$ (left), the single-block model (middle), and the two-block model (right).
    The first row shows the resulting function and the second row shows the partial derivative with respect to $x$.}
    \label{fig:prediction}
\end{figure}

\printbibliography

\appendix

\section{Appendix}
\begin{proposition}[\(L^{p_1,p_2}(U_1,U_2)\subseteq L^{q_1,q_2}(U_1,U_2)\)]
\label{prop:Lpp_inclusion}
    For \(i\in\{1,2\}\) let \(d_i\in\nn{}\),
    \(U_i\subset\rr{d_i}\) bounded, \(q_i\geq p_i\),
    and \(s_i:=\frac{q_i}{q_i-p_i}\) (\(s_i=\infty\) if \(q_i=p_i\)),
    then for $f\in L^{q_1,q_2}$
    \begin{align*}
        \norm{(\chF{U_1}\otimes\chF{U_2})f}{L^{p_1,p_2}}
        \leq |U_1|^{\frac{1}{p_1s_1}}|U_2|^{\frac{1}{p_2s_2}}\norm{f}{L^{q_1,q_2}}.
    \end{align*}
\end{proposition}
The proof of this statement is essentially, the same as for a single block.
The argument is simply applied twice.
\begin{proof}
    The choice of \(q_i\) and \(s_i\) guarantees that \(\frac{q_i}{p_i}\geq 1\), \(s_i\geq 1\), and \(\frac{p_i}{q_i}+\frac{1}{s_i}=1\).
    Therefore, by applying Hölder's inequality for a single block twice,
    \begin{subequations}
    \begin{align*}
        \norm{(\chF{U_1}\otimes\chF{U_2})f}{L^{p_1,p_2}}
        &=
        \normbig{
            \chF{U_1}\cdot
            \norm{
                \chF{U_2}f^{p_2}
                }{L^1}^{\frac{p_1}{p_2}}
            }{L^1}^{\frac{1}{p_1}}
        \leq
        \normbig{
            \chF{U_1}\cdot
            \norm{
                \chF{U_2}
                }{L^{s_2}}^{\frac{p_1}{p_2}}\cdot
            \norm{
                f^{p_2}
                }{L^{\frac{q_2}{p_2}}}^{\frac{p_1}{p_2}}
            }{L^1}^{\frac{1}{p_1}}
            \\
        &=
        \norm{
            \chF{U_2}
            }{L^{s_2}}^{\frac{1}{p_2}}\cdot
        \norm{
            \chF{U_1}\cdot
            \norm{f}{L^{q_2}}^{p_1}
            }{L^1}^{\frac{1}{p_1}}
        \leq
        \norm{
            \chF{U_1}
            }{L^{s_1}}^{\frac{1}{p_1}}\cdot
        \norm{
            \chF{U_2}
            }{L^{s_2}}^{\frac{1}{p_2}}\cdot
        \normbig{
            \norm{f}{L^{q_2}}^{p_1}
            }{L^{\frac{q_1}{p_1}}}^{\frac{1}{p_1}}
            \\
        &=|U_1|^{\frac{1}{p_1s_1}}|U_2|^{\frac{1}{p_2s_2}}\norm{f}{L^{q_1,q_2}}.
    \end{align*}
    \end{subequations}
\end{proof}

\begin{proposition}[Convergence of Smoothing for two Blocks]
\label{prop:smoothing_convergence}
    Let $d_1,d_2\in\nn{}$, $1\leq p_1,p_2\leq\infty$, $U\subset\rrI$ and $V\subset\rrII$ be bounded, and $h:\rrIxII\to\rr{}$ be locally integrable.
    Then,
    \begin{align*}
        \lim_{\epsilon_1\to0}&\lim_{\epsilon_2\to0}\norm{\chEF{U}{\epsilon_1}\chEF{V}{\epsilon_2}h}{L^{p_1,p_2}}
        =\norm{\chF{U}\chF{V}h}{L^{p_1,p_2}}.
    \end{align*}
\end{proposition}
\begin{proof}
    Using the smoothed characteristic functions \(\chEF{U}{\epsilon_1}\) and \(\chEF{V}{\epsilon_2}\),
    the extended sets $U_1$ and $U_2$,
    and $H(x)=\int_{\rrII} g(x,y)dy$
    we define\footnote{Note, that the definitions of $f$, $f_{\epsilon_1}$, and $\tilde{f}$ are deliberately based on $g$ and not on its variants.}
    \begin{align*}
        g(x,y)&:=|\ch{V}{y}h(x,y)|^{p_2},
        &
        f(x)&:=\ch{U}{x}^{p_1}H(x)^{\frac{p_1}{p_2}},
        \\
        g_{\epsilon_2}(x,y)&:=|\chE{V}{\epsilon_2}{y}h(x,y)|^{p_2},
        &
        f_{\epsilon_1}(x)&:=\chE{U}{\epsilon_1}{x}^{p_1}H(x)^{\frac{p_1}{p_2}},
        \\
        \tilde{g}(x,y)&:=|\ch{V_1}{y}h(x,y)|^{p_2},
        &
        \tilde{f}(x)&:=\ch{U_1}{x}^{p_1}H(x)^{\frac{p_1}{p_2}}.
    \end{align*}
    All these functions are integrable.
    Further, $\tilde{g}\geq |g_{\epsilon_2}|, |g|$ everywhere for all $0<\epsilon_2\leq 1$, and $g_{\epsilon_2}\to_{\epsilon_2\to0}g$ pointwise by \cite[Lemma 3.2]{Tartar07IntroductionSobolevSpaces}\footnote{The lemma guarantees uniform convergence which includes pointwise convergence, which is the necessary condition for the dominated convergence theorem.}.
    Analogous relations hold for $\tilde{f}$, $f_{\epsilon_1}$, $f$, and $\epsilon_1$, respectively.
    By the dominated convergence theorem we obtain the limit
    \begin{align*}
        \lim_{\epsilon_1\to0}\int_{\rrI}f_{\epsilon_1}(x)dx&=\int_{\rrI}f(x)dx
        \intertext{and the pointwise limit (for all $x\in\rrI$)}
        \lim_{\epsilon_2\to0}\int_{\rrII}g_{\epsilon_2}(x,y)dy&=\int_{\rrII}g(x,y)dy.
    \end{align*}
    Therefore,
    \begin{align*}
        \lim_{\epsilon_1\to0}&\lim_{\epsilon_2\to0}\norm{\chEF{U}{\epsilon_1}\chEF{V}{\epsilon_2}h}{L^{p_1,p_2}}\\
        &=\lim_{\epsilon_1\to0}\lim_{\epsilon_2\to0}\left(\int_{\rrI}\left(\int_{\rrII}|\chE{U}{\epsilon_1}{x}\chE{V}{\epsilon_2}{y}h(x,y)|^{p_2}dy\right)^{\frac{p_1}{p_2}}dx\right)^{\frac{1}{p_1}}\\
        &=\lim_{\epsilon_1\to0}\left(\int_{\rrI}\chE{U}{\epsilon_1}{x}^{p_1}\lim_{\epsilon_2\to0}\left(\int_{\rrII}|\chE{V}{\epsilon_2}{y}h(x,y)|^{p_2}dy\right)^{\frac{p_1}{p_2}}dx\right)^{\frac{1}{p_1}}\\
        &=\lim_{\epsilon_1\to0}\left(\int_{\rrI}\chE{U}{\epsilon_1}{x}^{p_1}\lim_{\epsilon_2\to0}\left(\int_{\rrII}g_{\epsilon_2}(x,y)dy\right)^{\frac{p_1}{p_2}}dx\right)^{\frac{1}{p_1}}\\
        &=\lim_{\epsilon_1\to0}\left(\int_{\rrI}\chE{U}{\epsilon_1}{x}^{p_1}\left(\int_{\rrII}g(x,y)dy\right)^{\frac{p_1}{p_2}}dx\right)^{\frac{1}{p_1}}\\
        &=\lim_{\epsilon_1\to0}\left(\int_{\rrI}f_{\epsilon_1}(x)dx\right)^{\frac{1}{p_1}}\\
        &=\left(\int_{\rrI}f(x)dx\right)^{\frac{1}{p_1}}\\
        &=\norm{\chF{U}\chF{V}h}{L^{p_1,p_2}}
    \end{align*}
\end{proof}

The following result is of independent interest,  where we study embedding
result in the setting of Fourier Lebesgue spaces.
\begin{lemma}[Fourier Lebesgue embedding]
\label{lem:weighted_FLt_in_FL1}
    For \(i\in\{1,2\}\) let \(d_i\in\nn{}\), \(t_i\in[1,2]\), and \(\vartheta_i\) be a weight function such that \(1/\vartheta_i\in L^1(\rr{d_i})\)
    and that there exists $k>0$ such that
    \(\vartheta_i(x_i)>k\) for all \(x_i\in\rr{d_i}\).
    For any weight function \(\omega(x_1, x_2)\) over \(\rrIxII\)
    elliptic with respect to \(\vartheta_1(x_1)\vartheta_2(x_2)\),
    it holds
    \begin{align*}
        L^1\cap\mathscr{F}L^{t_1,t_2}(\omega)\subseteq L^1\cap\mathscr{F}L^1.
    \end{align*}
\end{lemma}

\begin{proof}
    For the given weight function we have
    \begin{align*}
        \norm{{1}/{\omega}}{L^{s_1,s_2}}
        \lesssim \norm{\frac{1}{\vartheta_1(\cdot)\vartheta_2(\cdot\cdot)}}{L^{s_1,s_2}}
        =\norm{{1}/{\vartheta_1}}{L^{s_1}}\norm{{1}/{\vartheta_2}}{L^{s_2}}.
    \end{align*}
    Due to the lower bound on \(\vartheta_i\) we know that \(1/\vartheta_i\) is upper bounded, which allows us to provide an upper bound to this expression by the \(L^1\)-norm (with some multiplicative constant) whenever \(s_1,s_2\geq 1\).
    This is due to the monotonicity of the integral.
    Thus, by the assumption on \(\vartheta_1\) and \(\vartheta_2\) we get that \(1/\omega\in L^{s_1,s_2}\) for \(s_1,s_2\geq 1\).
    
    For \(a\in L^1\cap\mathscr{F}L^{t_1,t_2}(\omega)\) we can then apply Hölder's inequality with \(s_i=\frac{t_i}{t_i-1}\geq 2\) and get
    \begin{align*}
        \norm{a}{\mathscr{F}L^1}
        &=\norm{\hat{a}}{L^1}
        =\norm{\frac{1}{\omega}\omega\hat{a}}{L^1}
        \leq \norm{{1}/{\omega}}{L^{s_1,s_2}}\norm{\omega\hat{a}}{L^{t_1,t_2}}
        <\infty.
    \end{align*}
\end{proof}

\end{document}